\documentclass[11pt,a4paper]{article}
\usepackage[dvipsnames]{xcolor}
\usepackage[hyperref]{emnlp-ijcnlp-2019}
\usepackage{times}
\usepackage{etoolbox}
\usepackage{amsmath,amssymb}
\usepackage{amsthm}
\usepackage{latexsym}
\usepackage{inconsolata}
\usepackage{bm}
\usepackage{nicefrac}
\usepackage{mathtools}
\usepackage{empheq}
\usepackage{url}
\usepackage{microtype}
\usepackage{xspace}
\usepackage{enumitem}
\usepackage{booktabs}
\usepackage{algorithm}
\usepackage{setspace}
\usepackage{algorithmicx}
\usepackage{stmaryrd}
\usepackage{subcaption}
\usepackage[noend]{algpseudocode}
\algrenewcommand\alglinenumber[1]{{\textcolor{gray}{\sf\scriptsize#1}}}
\algrenewcommand\algorithmicindent{1.3em}%
\usepackage{soul}
\usepackage{marvosym}
\newcommand\markUnbabel{\Cancer}
\newcommand\markIT{\Leo}

\newtheorem{definition}{Definition}
\newtheorem{proposition}{Proposition}
\newtheorem{lemma}{Lemma}

\definecolor{vladcol}{rgb}{0.56, 0.27, 0.52}

\newtoggle{comms}

\toggletrue{comms}

\iftoggle{comms}{
\newcommand{\andre}[1]{\textcolor{purple}{[AM: #1]}}
\newcommand{\vlad}[1]{\textcolor{vladcol}{[VN: #1]}}
\newcommand{\goncalo}[1]{\textcolor{blue}{[GC: #1]}}
}{
\newcommand{\andre}[1]{}
\newcommand{\vlad}[1]{}
\newcommand{\goncalo}[1]{}
}

\newcommand*{\wrt}{\textit{w.\hspace{.07em}r\hspace{.07em}.t.}\@\xspace}
\newcommand*{\eg}{\textit{e.\hspace{.07em}g.}\@\xspace}
\newcommand*{\ie}{\textit{i.\hspace{.07em}e.}\@\xspace}

\newcommand*{\lhs}{\textit{l.\hspace{.07em}h\hspace{.07em}.s.}\@\xspace}
\newcommand*{\rhs}{\textit{r.\hspace{.07em}h\hspace{.07em}.s.}\@\xspace}

\newcommand{\tr}{\top}
\newcommand{\amap}{\bm{\pi}}
\newcommand{\reals}{\mathbb{R}}
\newcommand{\pfrac}[2]{\frac{\partial #1}{\partial #2}}
\newcommand{\simplex}{\triangle}
\newcommand\pp{p}
\newcommand\p{\bm{\pp}}
\newcommand\xx{z}
\newcommand\x{\bm{\xx}}
\newcommand{\figref}[1]{Figure~\ref{fig:#1}}
\newcommand{\eqnref}[1]{Equation~\ref{eq:#1}}
\newcommand{\HHs}{\mathsf{H}^\textsc{S}}

\newcommand{\HHt}{\mathsf{H}^{\textsc{T}}}
\newcommand{\ones}{\bm{1}}
\newcommand{\EE}{\mathbb{E}}

\newcommand{\secref}[1]{Section~\ref{sec:#1}}
\newcommand{\appref}[1]{Appendix~\ref{sec:#1}}

\newcommand\thresh{\tau}

\DeclareMathOperator*{\sigmoid}{\mathsf{sigmoid}}
\DeclareMathOperator*{\argmax}{\mathsf{argmax}}
\DeclareMathOperator*{\argmin}{\mathsf{argmin}}
\DeclareMathOperator*{\diag}{\mathsf{diag}}

\DeclareMathOperator*{\softmax}{\mathsf{softmax}}
\DeclareMathOperator*{\sparsemax}{\mathsf{sparsemax}}

\newcommand*\entmaxtext{entmax\xspace}
\DeclareMathOperator*{\entmax}{\mathsf{\entmaxtext}}
\newcommand*\aentmax[1][\alpha]{\mathop{\mathsf{#1}\textnormal{-}\mathsf{\entmaxtext}}}

\DeclareMathOperator{\att}{\mathsf{Att}}
\DeclareMathOperator{\ath}{\mathsf{Head}}

\newcommand{\langp}[2]{\textsc{#1}$\shortrightarrow$\textsc{#2}}

\aclfinalcopy 

\title{Adaptively Sparse Transformers}
\author{
Gon\c{c}alo M. Correia\textsuperscript{\markIT{}} \\
\href{mailto:goncalo.correia@lx.it.pt}{\tt goncalo.correia@lx.it.pt}
\And
Vlad Niculae\textsuperscript{\markIT{}} \\
\href{mailto:vlad@vene.ro}{\tt vlad@vene.ro} \\[1ex]
\textsuperscript{\markIT{}}Instituto de Telecomunica\c{c}\~oes, Lisbon,
Portugal\\
\textsuperscript{\markUnbabel}Unbabel, Lisbon, Portugal\\
\And
Andr\'e F.T. Martins\textsuperscript{\markIT{} \markUnbabel{}} \\
\href{mailto:andre.martins@unbabel.com}{\tt andre.martins@unbabel.com}\\
}

\date{}

\begin{document}

\maketitle%

\begin{abstract}
Attention mechanisms have become ubiquitous in NLP. Recent
architectures, notably the Transformer, learn powerful context-aware
word representations through layered, multi-headed attention. The
multiple heads learn diverse types of word relationships. However,
with standard softmax attention, all attention heads are dense,
assigning a non-zero weight to all context words.
%
%
In this work, we introduce the adaptively sparse Transformer, wherein
attention heads have flexible, context-dependent sparsity patterns.
This sparsity is accomplished by replacing softmax with
$\alpha$-\entmaxtext{}: a differentiable generalization of softmax
that allows low-scoring words to receive precisely zero weight.
Moreover, we derive a method to automatically learn the $\alpha$
parameter -- which controls the shape and sparsity of
$\alpha$-\entmaxtext{} -- allowing attention heads to choose between
focused or spread-out behavior.
Our adaptively sparse Transformer improves interpretability and head
diversity when compared to softmax Transformers on machine
translation datasets. Findings of the quantitative and qualitative
analysis of our approach include that heads in different layers learn
different sparsity preferences and tend to be more diverse in their
attention distributions than softmax Transformers. Furthermore, at no
cost in accuracy, sparsity in attention heads helps to uncover
different head specializations.
\end{abstract}%

\section{Introduction}%

The Transformer architecture~\citep{vaswani2017attention} for deep
neural networks has quickly risen to prominence in NLP through its
efficiency and performance, leading to improvements in the state of
the art of Neural Machine Translation
\citep[NMT;][]{marian,ott2018scaling}, as well as inspiring other
powerful general-purpose models like BERT~\citep{devlin2018bert} and
\mbox{GPT-2}~\citep{radford2019language}. At the heart of the
Transformer lie \emph{multi-head attention} mechanisms: each word is
represented by multiple different weighted averages of its relevant
context. As suggested by recent works on interpreting attention head
roles, separate attention heads may learn to look for various
relationships between tokens~\citep{tang2018why,raganato2018analysis,
marecek-rosa-2018-extracting,bert-rediscovers,specialized}.

\begin{figure}[t]
    \centering
    \includegraphics[width=\columnwidth]{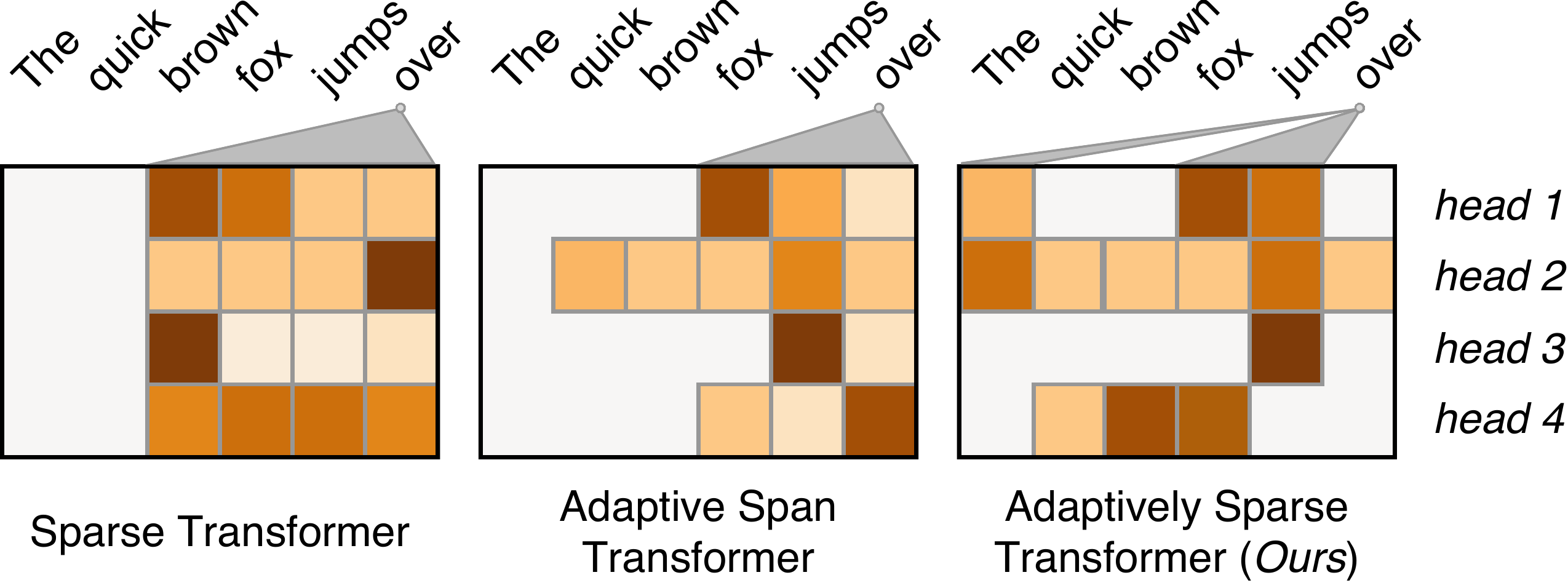}
    \caption{Attention distributions of different self-attention heads for the
    time step of the token ``over'', shown to compare our model to other
    related work. While the sparse
    Transformer~\citep{openai_sparse_transf} and the adaptive span
    Transformer~\citep{Sukhbaatar2019} only attend to words within a
    contiguous span of the past tokens, our model is not only able to
    obtain different and not necessarily contiguous sparsity patterns for
    each attention head, but is also able to tune its support over which
    tokens to attend adaptively.}
    \label{fig:comparison}
\end{figure}

The attention distribution of each head is predicted typically using
the \textbf{softmax} normalizing transform. As a result, all context
words have non-zero attention weight. Recent work on single attention
architectures suggest that using sparse normalizing transforms in
attention mechanisms such as sparsemax -- which can yield exactly
zero probabilities for irrelevant words -- may improve performance
and interpretability~\citep{malaviya2018sparse,deng2018latent,entmax}. Qualitative
analysis of attention heads \citep[Figure~5]{vaswani2017attention}
suggests that, depending on what phenomena they capture, heads tend
to favor flatter or more peaked distributions.

Recent works have proposed sparse Transformers~\citep{openai_sparse_transf} and adaptive span
Transformers~\citep{Sukhbaatar2019}. However,
the ``sparsity" of those models only limits the attention to a
contiguous span of past tokens, while in this work
we propose a \textbf{highly adaptive} Transformer model
that is capable of attending to a sparse set of words that are not
necessarily contiguous. \figref{comparison} shows the relationship of
these methods with ours.

Our contributions are the following:

\begin{itemize}
\item We introduce \textbf{sparse attention} into the
Transformer architecture, showing that it eases
interpretability and leads to slight accuracy gains.
\item We propose an adaptive version of sparse attention,
where the shape of each attention head is {\bf learnable} and can vary continuously and
dynamically between the dense limit case of \emph{softmax} and the sparse,
piecewise-linear \emph{sparsemax} case.\footnote{
Code and pip package available at \url{https://github.com/deep-spin/entmax}.}
\item We make an extensive analysis of the added interpretability of these
models, identifying both crisper examples of attention head behavior observed in
previous work, as well as novel behaviors unraveled thanks to the sparsity
and adaptivity of our proposed model.
\end{itemize}

\section{Background}
\subsection{The Transformer}

In NMT, the Transformer~\citep{vaswani2017attention} is a
sequence-to-sequence (seq2seq) model which maps an input sequence to
an output sequence through hierarchical \textbf{multi-head attention}
mechanisms, yielding a dynamic, context-dependent strategy for
propagating information within and across sentences. It contrasts
with previous seq2seq models, which usually rely either on costly
gated recurrent operations \citep[often
LSTMs:][]{bahdanau2014neural,luong2015effective} or static
convolutions~\citep{convseq}.

Given $n$ query contexts and $m$ sequence items under consideration,
attention mechanisms compute, for each query, a weighted
representation of the items. The particular attention mechanism used
in \citet{vaswani2017attention} is called \emph{scaled dot-product
attention}, and it is computed in the following way:
\begin{equation}
    \att(\bm{Q}, \bm{K}, \bm{V}) = \amap
\left(\frac{\bm{Q}\bm{K}^\tr}{\sqrt{d}}\right) \bm{V},
    \label{eq:att_scaled_dot}
\end{equation}
where $\bm{Q} \in \reals^{n \times d}$ contains representations of the
queries, $\bm{K}, \bm{V} \in \reals^{m \times d}$
are the \emph{keys} and \emph{values} of the items attended over,
and $d$ is the dimensionality of these
representations.
The $\amap$ mapping normalizes row-wise using \textbf{softmax},
$\amap(\bm{Z})_{ij} = \softmax(\bm{z}_i)_j$, where
\begin{equation}\label{eq:softmax}
    \softmax(\bm{z}) = \frac{\exp(z_j)}{\sum_{j'} \exp(z_{j'})}.
\end{equation}
In words, the \emph{keys} are used to compute a relevance score
between each item and query. Then, normalized attention weights are computed
using softmax, and these are used to weight the \emph{values} of each item at each
query context.

However, for complex tasks, different parts of a sequence may be relevant in
different ways, motivating \emph{multi-head attention} in Transformers.
This is simply the application of
Equation~\ref{eq:att_scaled_dot} in parallel $H$ times, each with a different,
learned linear transformation that allows specialization:
\begin{equation}\label{eq:head}%
\hspace{-.01ex}%
\ath_i(\bm{Q}\!,\!\bm{K}\!,\!\bm{V})\!=\!\att(\bm{QW}_i^Q\!\!,\bm{KW}_i^K\!\!,\bm{VW}_i^V\!)%
\hspace{-1.5ex}%
\end{equation}
In the Transformer, there are three separate multi-head attention mechanisms for
distinct purposes:
\begin{itemize}
\item \textbf{Encoder self-attention:} builds rich, layered representations of
each input word, by attending on the entire input sentence.
\item \textbf{Context attention:} selects
a representative weighted average of the encodings of the input words, at each
time step of the decoder.
\item \textbf{Decoder self-attention:} attends over the partial output sentence
fragment produced so far.
\end{itemize}
Together, these mechanisms enable the contextualized flow of information between
the input sentence and the sequential decoder.

\subsection{Sparse Attention}

The softmax mapping (Equation~\ref{eq:softmax}) is elementwise
proportional to $\exp$, therefore it can never assign a weight of
\textbf{exactly zero}. Thus, unnecessary items are still taken into
consideration to some extent. Since its output sums to one, this
invariably means less weight is assigned to the relevant items,
potentially harming performance and
interpretability~\citep{jain2019attention}. This has motivated a line
of research on learning networks with \emph{sparse}
mappings~\citep{sparsemax,fusedmax,louizos,shao2019ssn}. We focus on
a recently-introduced flexible family of transformations,
$\alpha$-\entmaxtext~\citep{blondel2019learning,entmax}, defined as:
\begin{equation}\label{eq:define_entmax}
    \aentmax(\bm{z}) \coloneqq
    \argmax_{\p \in \simplex^d} \bm{p}^\top\bm{z} + \HHt_{\alpha}(\bm{p}),
\end{equation}
where $\simplex^d \coloneqq
\{\p \in \reals^d:\sum_{i} \pp_i = 1\}$
is the \emph{probability simplex}, and, for $\alpha\geq1$,
$\HHt_\alpha$ is the Tsallis continuous family of entropies
\citep{Tsallis1988}:
\begin{equation}\label{eq:tsallisdef}
    \HHt_{\alpha}(\bm{p})\!\coloneqq\!
\begin{cases}
\frac{1}{\alpha(\alpha-1)}\sum_j\!\left(p_j - p_j^\alpha\right)\!, &
\!\!\!\alpha \ne 1,\\
-\sum_j \pp_j \log \pp_j, &
\!\!\!\alpha = 1.
\end{cases}
\end{equation}
This family contains the well-known Shannon and Gini entropies,
corresponding to the cases $\alpha=1$ and $\alpha=2$, respectively.

\eqnref{define_entmax} involves a convex optimization subproblem. Using the
definition of $\HHt_\alpha$, the optimality conditions may be used to derive the
following form for the solution (\appref{bgform}):
\begin{equation}\label{eq:entmax_form}
\aentmax(\x) = [(\alpha - 1){\x} - \thresh \ones]_+^{\nicefrac{1}{\alpha-1}},
\end{equation}
where $[\cdot]_+$ is
the positive part (ReLU) function, $\bm{1}$ denotes the vector of all ones, and
$\thresh$ -- which acts like a threshold -- is the Lagrange multiplier
corresponding to the $\sum_i \pp_i=1$ constraint.

\paragraph{Properties of {\boldmath $\alpha$}-\entmaxtext.}
The appeal of $\alpha$-\entmaxtext for attention rests on the
following properties. For $\alpha=1$ (\ie, when $\HHt_\alpha$ becomes
the Shannon entropy), it exactly recovers the softmax mapping (We
provide a short derivation in \appref{softmax}.). For all $\alpha>1$
it permits sparse solutions, in stark contrast to softmax. In
particular, for $\alpha=2$, it recovers the sparsemax mapping
\citep{sparsemax}, which is piecewise linear. In-between, as $\alpha$
increases, the mapping continuously gets sparser as its curvature
changes.

To compute the value of $\alpha$-\entmaxtext, one must find the
threshold $\thresh$ such that the \rhs in \eqnref{entmax_form} sums
to one. \citet{blondel2019learning} propose a general bisection
algorithm. \citet{entmax} introduce a faster, exact algorithm for
$\alpha=1.5$, and enable using $\aentmax$ with fixed $\alpha$ within
a neural network by showing that the $\alpha$-\entmaxtext Jacobian
\wrt $\x$ for $\p^\star = \aentmax(\x)$ is

\begin{equation}
\begin{aligned}
\pfrac{\aentmax(\x)}{\x} = \diag(\bm{s}) - \frac{1}{\sum_j s_j} \bm{ss}^\top, \\
\text{where}\quad s_i = \begin{cases}(\pp_i^\star)^{2-\alpha}, & \pp_i^\star > 0, \\
0,& \pp_i^\star = 0. \\\end{cases}
\end{aligned}
\end{equation}

Our work furthers the study of $\alpha$-\entmaxtext by providing a
derivation of the Jacobian {\bf \wrt the hyper-parameter}
$\boldsymbol{\alpha}$ (\secref{adaptive}), thereby allowing the shape
and sparsity of the mapping to be learned automatically. This is
particularly appealing in the context of multi-head attention
mechanisms, where we shall show in \secref{stats} that different
heads tend to learn different sparsity behaviors.

\section{Adaptively Sparse Transformers\\\quad with {\boldmath $\alpha$}-\entmaxtext}
\label{sec:adaptive}

We now propose a novel Transformer architecture wherein we simply
replace softmax with $\alpha$-\entmaxtext{} in the attention heads.
Concretely, we replace the row normalization $\amap$ in
\eqnref{att_scaled_dot} by

\begin{equation}
    \amap(\bm{Z})_{ij} = \aentmax(\bm{z}_i)_j
\end{equation}

This change leads to sparse attention weights, as long as
$\alpha>1$; in particular, $\alpha=1.5$ is a sensible starting
point~\citep{entmax}.

\paragraph{Different {\boldmath $\alpha$} per head.}
Unlike LSTM-based seq2seq models, where $\alpha$ can be more easily
tuned by grid search, in a Transformer, there are many attention
heads in multiple layers. Crucial to the power of such models, the
different heads capture different linguistic phenomena, some of them
isolating important words, others spreading out attention across
phrases \citep[Figure~5]{vaswani2017attention}. This motivates using
different, adaptive $\alpha$ values for each attention head, such
that some heads may learn to be sparser, and others may become closer
to softmax. We propose doing so by treating the $\alpha$ values as
neural network parameters, optimized via stochastic gradients along
with the other weights.

\paragraph{Derivatives \wrt~{\boldmath $\alpha$}.}
In order to optimize $\alpha$ automatically via gradient methods, we
must compute the Jacobian of the \entmaxtext output \wrt $\alpha$.
Since \entmaxtext is defined through an optimization problem, this is
non-trivial and cannot be simply handled through automatic
differentiation; it falls within the domain of \emph{argmin
differentiation}, an active research topic in optimization
\citep{gould,optnet}.

One of our key contributions is the derivation of a closed-form
expression for this Jacobian. The next proposition provides
such an expression, enabling \entmaxtext layers with
adaptive $\alpha$. To the best of our knowledge, ours is the first
neural network module that can automatically, continuously vary in
shape away from softmax and toward sparse mappings like sparsemax.

\begin{proposition}\label{prop:grad_alpha}%
Let $\p^\star \coloneqq \aentmax(\x)$ be the solution of
\eqnref{define_entmax}.
Denote the distribution $\tilde{\pp}_i \coloneqq \nicefrac{(\pp_i^\star)^{2 - \alpha}}{
\sum_j(\pp_j^\star)^{2-\alpha}}$ and let
$h_i \coloneqq -\pp^\star_i \log \pp^\star_i$.
The $i$\textsuperscript{th} component of the Jacobian
$\bm{g} \coloneqq \pfrac{\aentmax(\x)}{\alpha}$ is
\begin{equation}\label{eq:final_gradient_alpha_supp}
g_i =%
\begin{cases}
    \frac{p_i^{\star} - \tilde{\pp}_i}{(\alpha-1)^2} +
\frac{h_i - \tilde{\pp}_i%
\sum_j h_j%
}{\alpha-1}, & \alpha > 1,\\[1.5ex]
    \frac{h_i \log \pp_i^{\star} - \pp_i^{\star} \sum_j h_j \log p_j^{\star}}{2}, & \alpha = 1.
\end{cases}
\end{equation}
\end{proposition}
\noindent%
The proof uses implicit function differentiation and is given in Appendix \ref{sec:alpha_grad}.

Proposition~\ref{prop:grad_alpha} provides the remaining missing
piece needed for training adaptively sparse Transformers. In the
following section, we evaluate this strategy on neural machine
translation, and analyze the behavior of the learned attention heads.

\section{Experiments}
We apply our adaptively sparse Transformers on four machine translation tasks.
For comparison, a natural baseline is the standard Transformer
architecture using the softmax transform in its multi-head attention mechanisms.
We consider two other model variants in our experiments that make use of different
normalizing transformations:

\begin{itemize}
\item \textbf{1.5-\entmaxtext:} a Transformer with sparse \entmaxtext
attention with fixed $\alpha=1.5$ for all heads. This is a novel model,
since 1.5-\entmaxtext{} had only been proposed for
RNN-based NMT models~\citep{entmax}, but never
in Transformers, where attention modules are not just one single
component of the seq2seq model but rather an integral part of all of
the model components.%
\item \textbf{\boldmath $\alpha$-\entmaxtext:} an \textbf{adaptive}
Transformer with sparse \entmaxtext attention with a different,
learned $\alpha_{i,j}^t$ for each head.
\end{itemize}

The adaptive model has an additional scalar parameter per attention head per
layer for each of the three attention mechanisms (encoder self-attention,
context attention, and decoder self-attention), \ie,
\begin{equation}
\begin{aligned}%
\hspace{-8pt}
\big \{ a_{i,j}^{t} \in \reals:~&
i \in \{1, \dots, L\},
j \in \{1, \dots, H\}, \\
& t \in \{\texttt{enc}, \texttt{ctx}, \texttt{dec}\} \big\},
\end{aligned}
\hspace{-5pt}
\end{equation}
and we set $\alpha_{i,j}^t = 1 + \sigmoid(a_{i,j}^t) \in ]1, 2[$.
All or some of the $\alpha$ values can be tied if desired, but we
keep them independent for analysis purposes.

\begin{table*}[ht]

    \begin{center}
    \small
    \begin{tabular}{lrrrr}
    \toprule
    activation
    & \langp{de}{en} & \langp{ja}{en}
    & \langp{ro}{en} & \langp{en}{de}\\
    \midrule
    $\softmax$
    & 29.79
    & 21.57
    & 32.70
    & 26.02 \\
    $\aentmax[1.5]$
    & 29.83
    & \textbf{22.13}
    & \textbf{33.10}
    & 25.89 \\
    $\aentmax[\alpha]$
    & \textbf{29.90}
    & 21.74
    & 32.89
    & \textbf{26.93} \\
    \bottomrule
    \end{tabular}
    \end{center}
    \caption{Machine translation tokenized BLEU test results
    on IWSLT 2017 \langp{de}{en},
    KFTT \langp{ja}{en}, WMT 2016 \langp{ro}{en} and
    WMT 2014 \langp{en}{de}, respectively.\label{table:mt}}
    \end{table*}

\paragraph{Datasets.} Our models were trained on 4 machine
translation datasets of different training sizes:

\begin{itemize}[itemsep=.5ex,leftmargin=2ex]
    \item IWSLT 2017 German $\rightarrow$ English
    \citep[\langp{de}{en},][]{cettolooverview}: ~200K sentence pairs.
    \item KFTT Japanese $\rightarrow$ English
    \citep[\langp{ja}{en},][]{neubig11kftt}: ~300K sentence pairs.
    \item WMT 2016 Romanian $\rightarrow$ English
    \citep[\langp{ro}{en},][]{bojar2016findings}: ~600K sentence pairs.
    \item WMT 2014 English $\rightarrow$ German
    \citep[\langp{en}{de},][]{bojar2014findings}: ~4.5M sentence pairs.
\end{itemize}

All of these datasets were preprocessed with byte-pair
encoding~\citep[BPE;][]{sennrich2016neural}, using joint
segmentations of 32k merge operations.

\paragraph{Training.}
We follow the dimensions of the Transformer-Base model of
\citet{vaswani2017attention}: The number of layers is $L=6$ and
number of heads is $H=8$ in the encoder self-attention, the context
attention, and the decoder self-attention. We use a mini-batch size
of 8192 tokens and warm up the learning rate linearly until 20k
steps, after which it decays according to an inverse square root
schedule. All models were trained until convergence of validation
accuracy, and evaluation was done at each 10k steps for
\langp{ro}{en} and \langp{en}{de} and at each 5k steps for
\langp{de}{en} and \langp{ja}{en}. The end-to-end computational
overhead of our methods, when compared to standard softmax, is
relatively small; in training tokens per second, the models using
$\alpha$-\entmaxtext and $1.5$-\entmaxtext are, respectively, $75\%$
and $90\%$ the speed of the softmax model.

\paragraph{Results.}
We report test set tokenized BLEU~\citep{papineni2002bleu} results in
Table \ref{table:mt}. We can see that replacing softmax by
\entmaxtext{} does not hurt performance in any of the datasets;
indeed, sparse attention Transformers tend to have slightly higher
BLEU, but their sparsity leads to a better potential for analysis. In
the next section, we make use of this potential by exploring the
learned internal mechanics of the self-attention heads.

\section{Analysis}

We conduct an analysis for the higher-resource dataset WMT 2014
English $\rightarrow$ German of the attention in the sparse adaptive
Transformer model ($\alpha$-\entmaxtext) at multiple levels: we
analyze high-level statistics as well as individual head behavior.
Moreover, we make a qualitative analysis of the interpretability
capabilities of our models.

\subsection{High-Level Statistics}
\label{sec:stats}

\paragraph{What kind of {\boldmath $\alpha$} values are learned?}
\figref{learning_alpha} shows the learning trajectories of the
$\alpha$ parameters of a selected subset of heads. We generally
observe a tendency for the randomly-initialized $\alpha$ parameters
to decrease initially, suggesting that softmax-like behavior may be
preferable while the model is still very uncertain. After around one
thousand steps, some heads change direction and become sparser,
perhaps as they become more confident and specialized. This shows
that the initialization of $\alpha$ does not predetermine its
sparsity level or the role the head will have throughout. In
particular, head $8$ in the encoder self-attention layer $2$ first
drops to around $\alpha=1.3$ before becoming one of the sparsest
heads, with $\alpha\approx2$.

The overall distribution of $\alpha$ values at convergence can be
seen in \figref{hist_alphas}. We can observe that the encoder
self-attention blocks learn to concentrate the $\alpha$ values in two
modes: a very sparse one around $\alpha \rightarrow 2$, and a dense
one between softmax and 1.5-\entmaxtext{}. However, the decoder self
and context attention only learn to distribute these parameters in a
single mode. We show next that this is reflected in the average
density of attention weight vectors as well.

\begin{figure}[t]
    \includegraphics[width=\columnwidth]{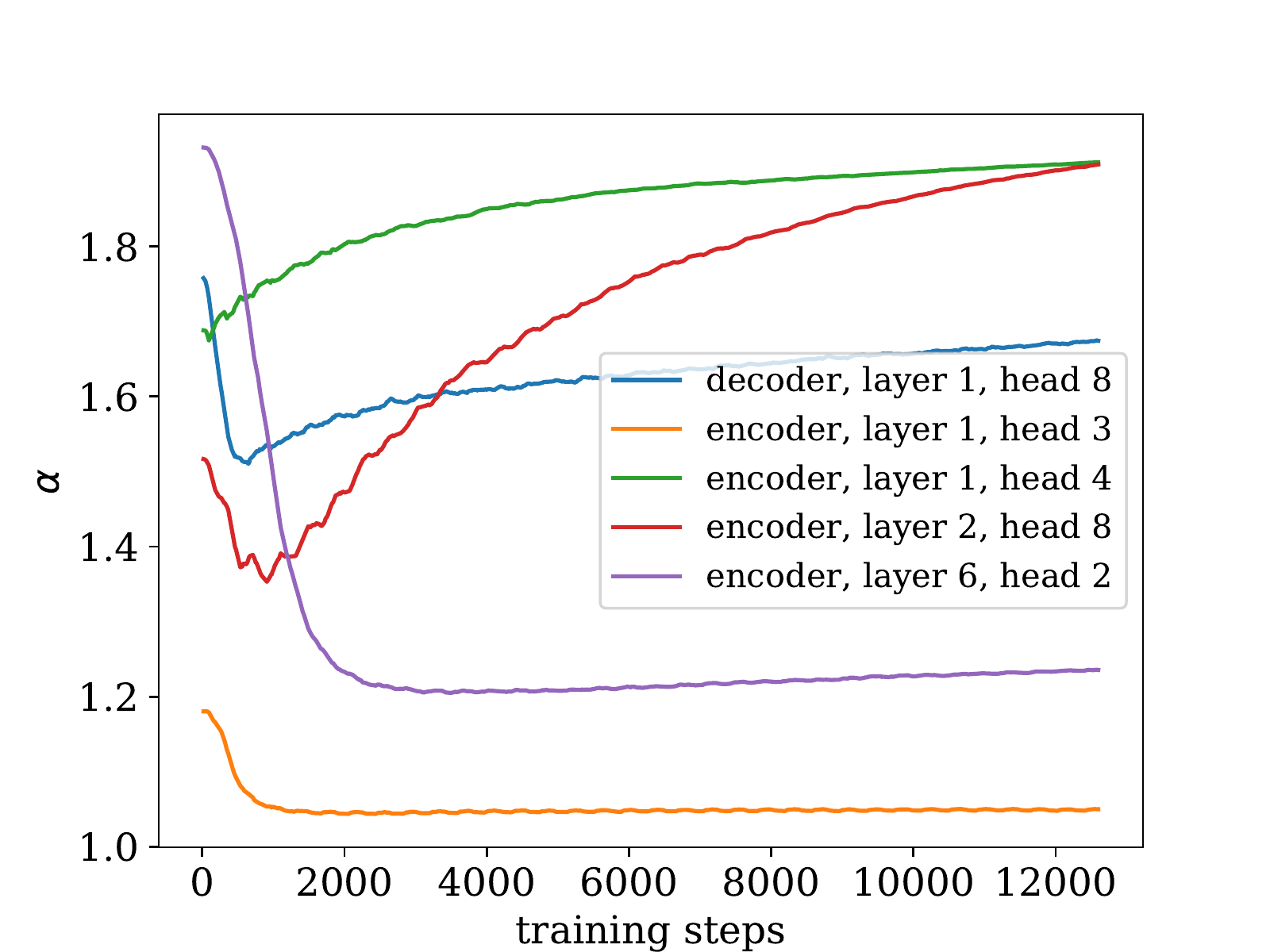}
    \caption{\label{fig:learning_alpha}
    Trajectories of $\alpha$ values for a subset of the heads during
    training. Initialized at random, most heads become denser in the
    beginning, before converging. This suggests that dense attention may
    be more beneficial while the network is still uncertain, being
    replaced by sparse attention afterwards.}
\end{figure}

\begin{figure}[t]
    \includegraphics[width=\columnwidth]{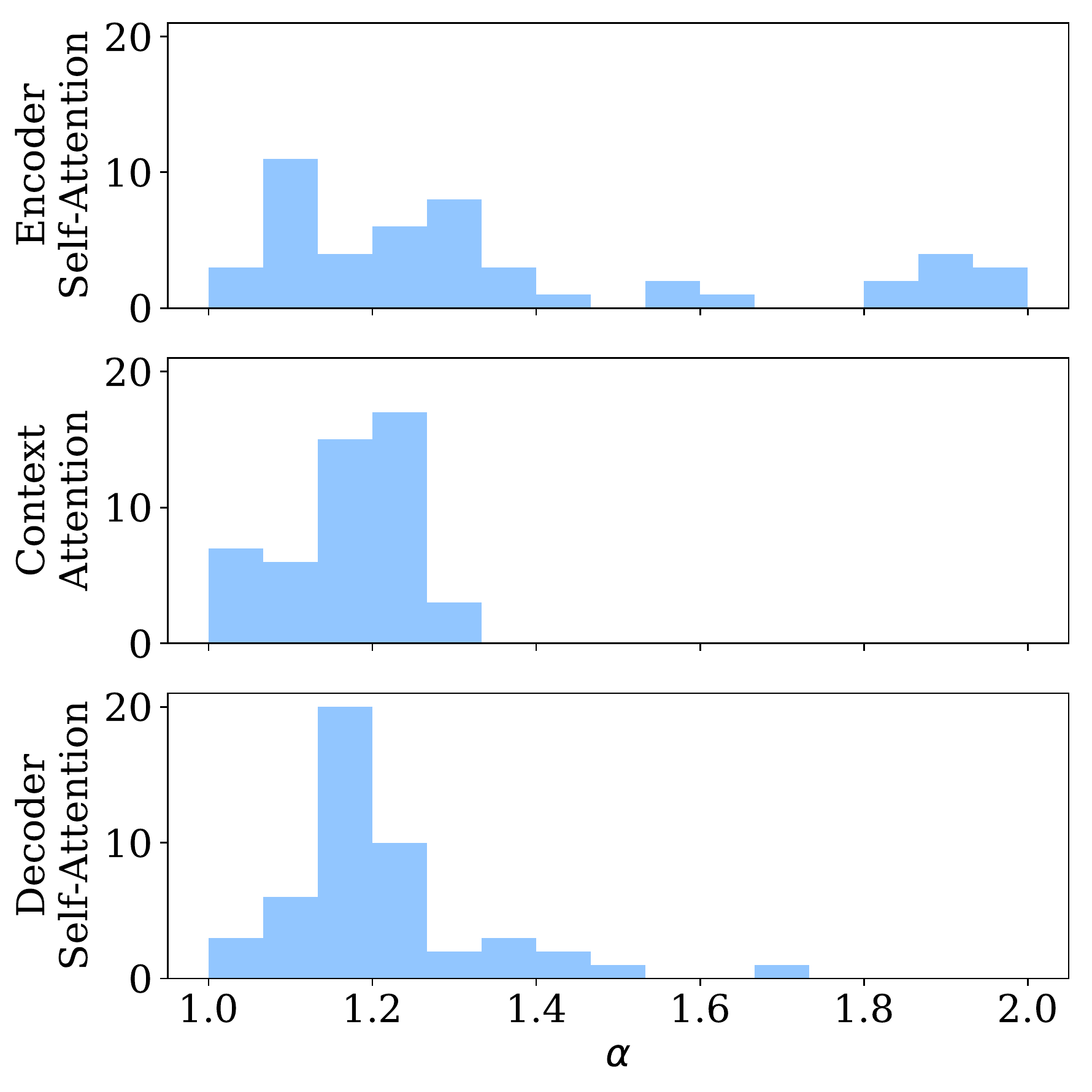}
    \caption{%
    \label{fig:hist_alphas}%
    Distribution of learned $\alpha$ values per attention block.
    While the encoder self-attention has a bimodal distribution
    of values of $\alpha$,
    the decoder self-attention and context attention have a single mode.}
\end{figure}

\paragraph{Attention weight density when translating.}
For any $\alpha>1$, it would still be possible for the weight
matrices in Equation~\ref{eq:head} to learn re-scalings so as to make
attention sparser or denser. To visualize the impact of adaptive
$\alpha$ values, we compare the empirical attention weight density
(the average number of tokens receiving non-zero attention) within
each module, against sparse Transformers with fixed $\alpha=1.5$.

\figref{hist_densities} shows that, with fixed $\alpha=1.5$, heads
tend to be sparse and similarly-distributed in all three attention
modules. With learned $\alpha$, there are two notable changes: (i) a
prominent mode corresponding to fully dense probabilities, showing
that our models learn to combine sparse and dense attention, and (ii)
a distinction between the encoder self-attention -- whose background
distribution tends toward extreme sparsity -- and the other two
modules, who exhibit more uniform background distributions. This
suggests that perhaps entirely sparse Transformers are suboptimal.

The fact that the decoder seems to prefer denser attention
distributions might be attributed to it being auto-regressive, only
having access to past tokens and not the full sentence. We speculate
that it might lose too much information if it assigned weights of
zero to too many tokens in the self-attention, since there are fewer
tokens to attend to in the first place.

Teasing this down into separate layers,
\figref{head_density_per_layer} shows the average (sorted) density of
each head for each layer. We observe that $\alpha$-\entmaxtext{} is
able to learn different sparsity patterns at each layer, leading to
more variance in individual head behavior, to clearly-identified
dense and sparse heads, and overall to different tendencies compared
to the fixed case of $\alpha=1.5$.

\begin{figure}[t]
    \centering
        \includegraphics[width=.95\columnwidth]{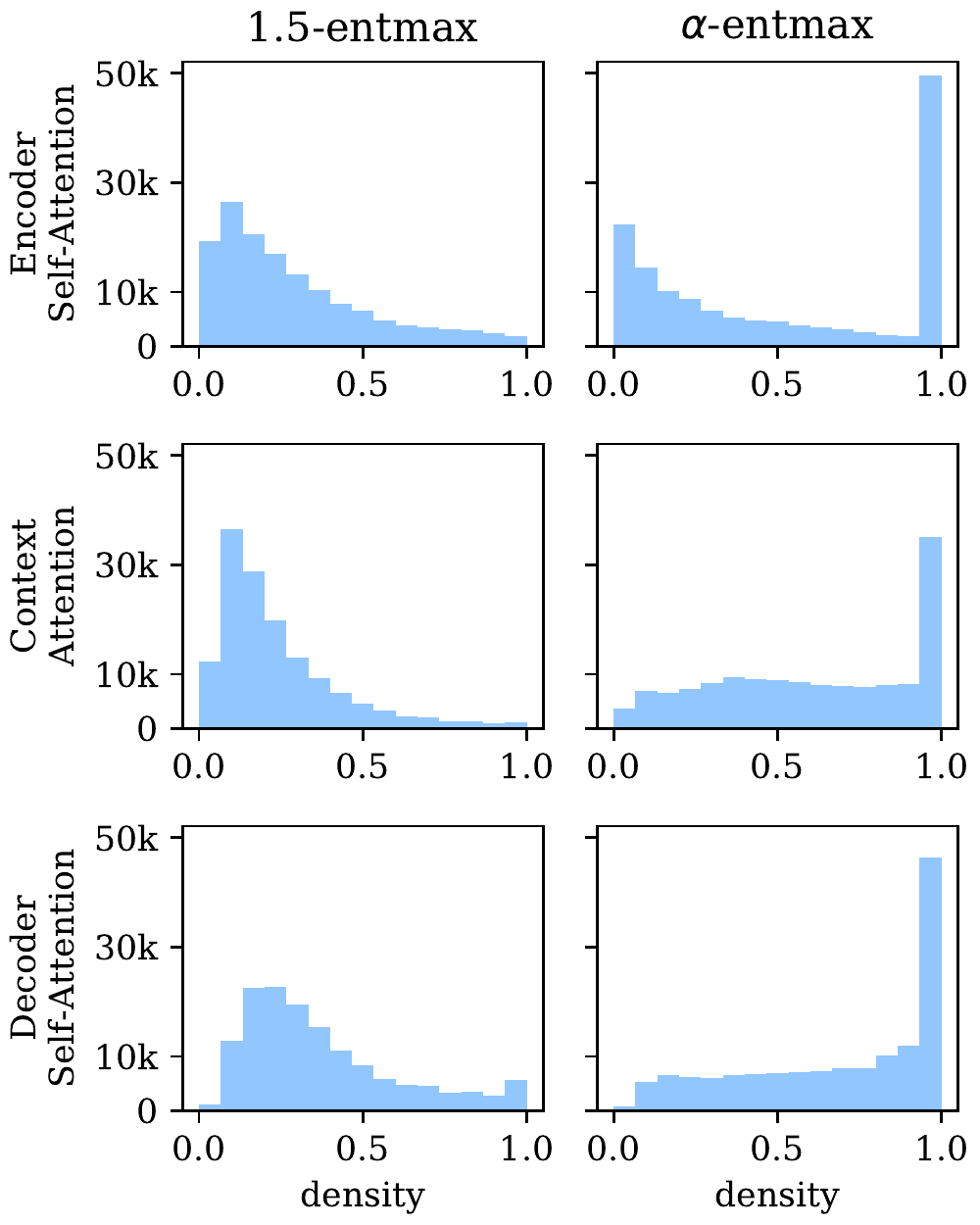}
        \caption{%
    \label{fig:hist_densities}
    Distribution of attention densities (average number of tokens
    receiving non-zero attention weight) for all attention heads and all
    validation sentences.
    When compared to 1.5-\entmaxtext{}, $\alpha$-\entmaxtext{}
    distributes the sparsity in a more uniform manner, with a clear mode
    at fully dense attentions, corresponding to the heads with low
    $\alpha$. In the softmax case, this distribution would lead to a
    single bar with density 1.}
\end{figure}

\begin{figure}[t]
    \centering
        \includegraphics[
            width=.95\columnwidth]{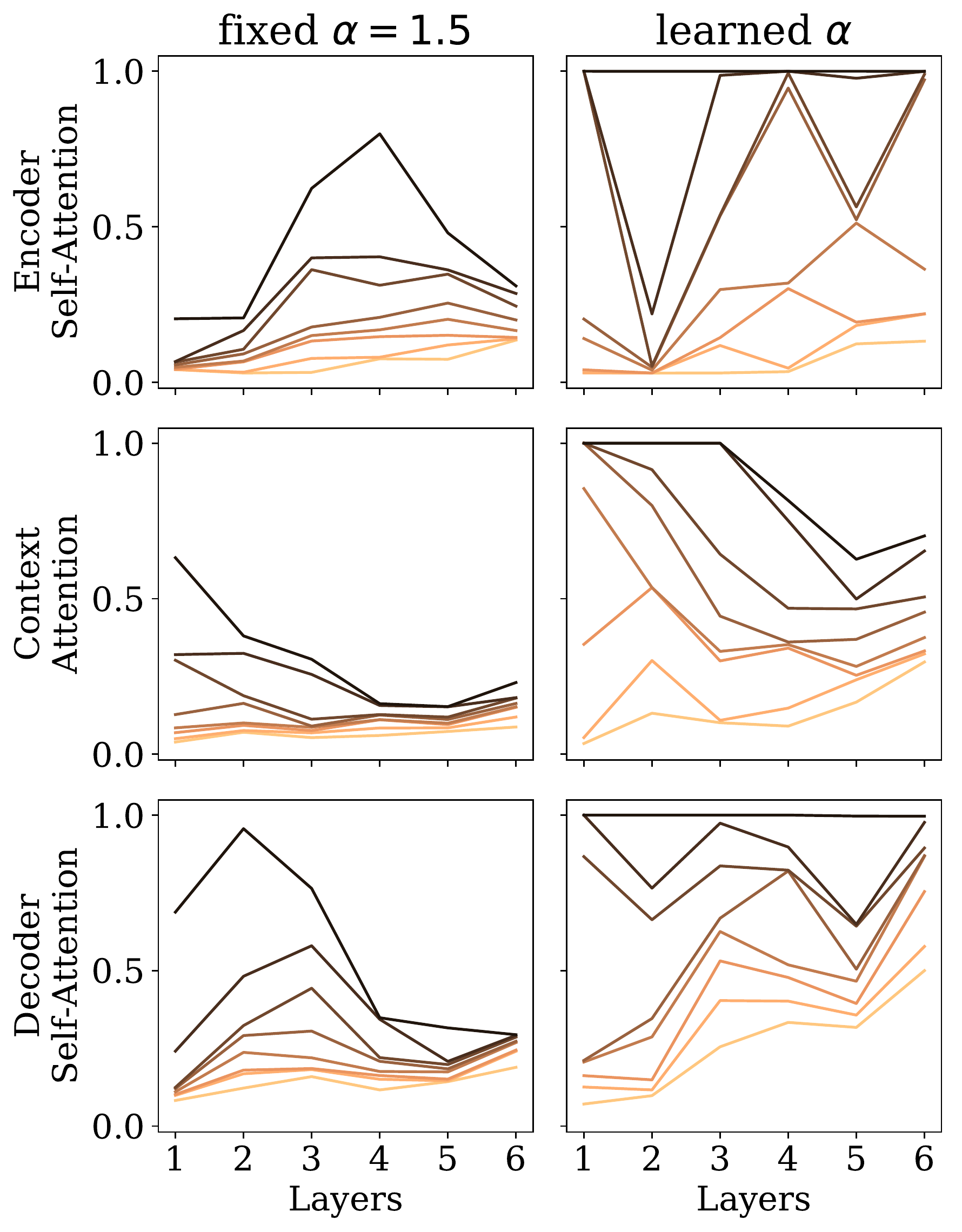}
        \caption{%
    \label{fig:head_density_per_layer}
    Head density per layer for fixed and learned $\alpha$. Each line
    corresponds to an attention head; lower values mean that that
    attention head is sparser. Learned $\alpha$ has higher variance.
    }
\end{figure}

\paragraph{Head diversity.}
To measure the overall disagreement between attention heads, as a measure of
head diversity, we use the following generalization of the Jensen-Shannon
divergence:

\begin{equation}
JS = \HHs\left(\frac{1}{H}\sum_{j=1}^H \bm{p}_{j}\right) - \frac{1}{H}\sum_{j=1}^{H}
\HHs(\bm{p}_j)
\end{equation}

where $\bm{p}_j$ is the vector of attention weights assigned by head
$j$ to each word in the sequence, and $\HHs$ is the Shannon entropy,
base-adjusted based on the dimension of $\bm{p}$ such that $JS \leq
1$. We average this measure over the entire validation set. The
higher this metric is, the more the heads are taking different roles
in the model.

\figref{js_divs}  shows that both sparse Transformer variants show more
diversity than the traditional softmax one. Interestingly, diversity seems to
peak in the middle layers of the encoder self-attention and context attention,
while this is not the case for the decoder self-attention.

The statistics shown in this section can be found for the other language pairs in
\appref{plots_lps}.

\begin{figure}[h]
    \includegraphics[width=\columnwidth]{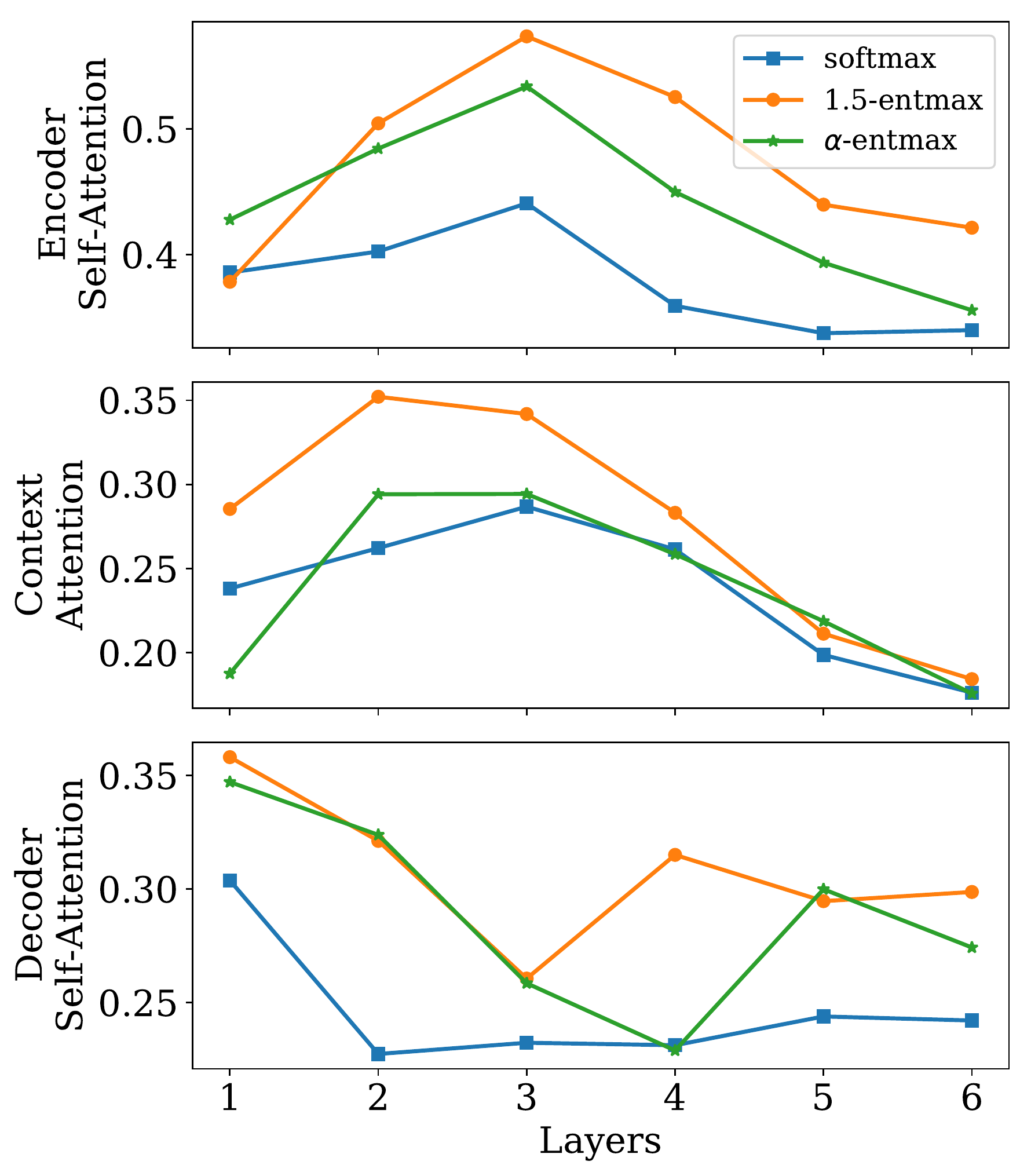}
    \caption{%
\label{fig:js_divs}
Jensen-Shannon Divergence between heads at each layer. Measures the
disagreement between heads: the higher the value, the more the heads
are disagreeing with each other in terms of where to attend. Models
using sparse \entmaxtext have more diverse attention than the softmax
baseline.
}
\end{figure}

\subsection{Identifying Head Specializations}

Previous work pointed out some specific roles played by
different heads in the softmax Transformer
model~\citep{voita2018context,tang2018why,specialized}. Identifying
the specialization of a head can be done by observing the
type of tokens or sequences that the head often assigns most
of its attention weight; this is facilitated by sparsity.

\paragraph{Positional heads.}
One particular type of head, as noted by \citet{specialized}, is the
positional head. These heads tend to focus their attention on either
the previous or next token in the sequence, thus obtaining
representations of the neighborhood of the current time step. In
\figref{head_prev}, we show attention plots for such heads, found for
each of the studied models. The sparsity of our models allows these
heads to be more confident in their representations, by assigning the
whole probability distribution to a single token in the sequence.
Concretely, we may measure a positional head's \textbf{confidence} as
the average attention weight assigned to the previous token. The
softmax model has three heads for position $-1$, with median
confidence $93.5\%$. The $1.5$-\entmaxtext model also has three heads
for this position, with median confidence $94.4\%$. The adaptive
model has four heads, with median confidences $95.9\%$, the
lowest-confidence head being dense with $\alpha=1.18$, while the
highest-confidence head being sparse ($\alpha=1.91$).

For position $+1$, the models each dedicate one head, with confidence
around $95\%$, slightly higher for \entmaxtext. The adaptive model
sets $\alpha=1.96$ for this head.

\begin{figure}[h]
    \includegraphics[width=\columnwidth]{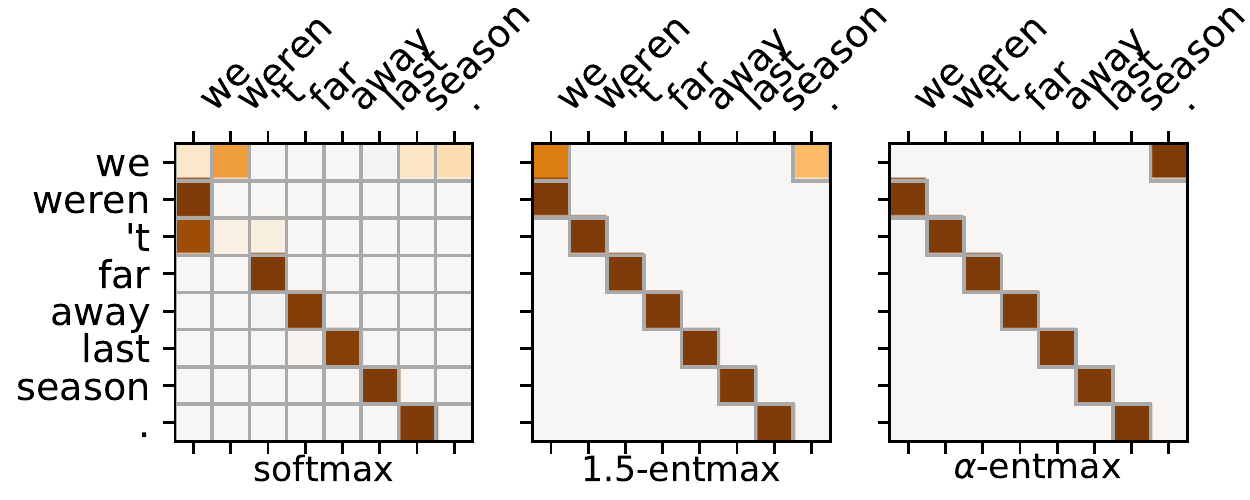}
\caption{
Self-attention from the most confidently previous-position head in
each model. The learned parameter in the $\alpha$-\entmaxtext model
is $\alpha=1.91$. Quantitatively more confident, visual inspection
confirms that the adaptive head behaves more consistently.}
\label{fig:head_prev}
\end{figure}

\paragraph{BPE-merging head.}
Due to the sparsity of our models, we are able to identify other head
specializations, easily identifying which heads should be further analysed.
In \figref{head_bpe} we show one such head where the $\alpha$ value
is particularly high (in the encoder, layer 1, head 4 depicted in
\figref{learning_alpha}). We found that this head most often looks at
the current time step with high confidence, making it a positional head
with offset $0$. However, this head often spreads weight sparsely
over 2-3 neighboring tokens, when the tokens are part of the same BPE
cluster\footnote{BPE-segmented words are denoted by $\sim$ in the
figures.} or hyphenated words. As this head is in the first layer, it
provides a useful service to the higher layers by combining
information evenly within some BPE clusters.

For each BPE cluster or cluster of hyphenated words,
we computed a score between 0 and 1 that corresponds to the
maximum attention mass assigned by any token to the rest of the
tokens inside the cluster in order to quantify the BPE-merging
capabilities of these heads.\footnote{If the
cluster has size 1, the score is the weight the token assigns to
itself.} There are not any attention heads in
the softmax model that are able to obtain a score over $80\%$, while
for $1.5$-\entmaxtext and $\alpha$-\entmaxtext there are two heads
in each ($83.3\%$ and $85.6\%$ for $1.5$-\entmaxtext and $88.5\%$ and
$89.8\%$ for $\alpha$-\entmaxtext).

\begin{figure}[t]
    \centering
    \includegraphics[width=\columnwidth]{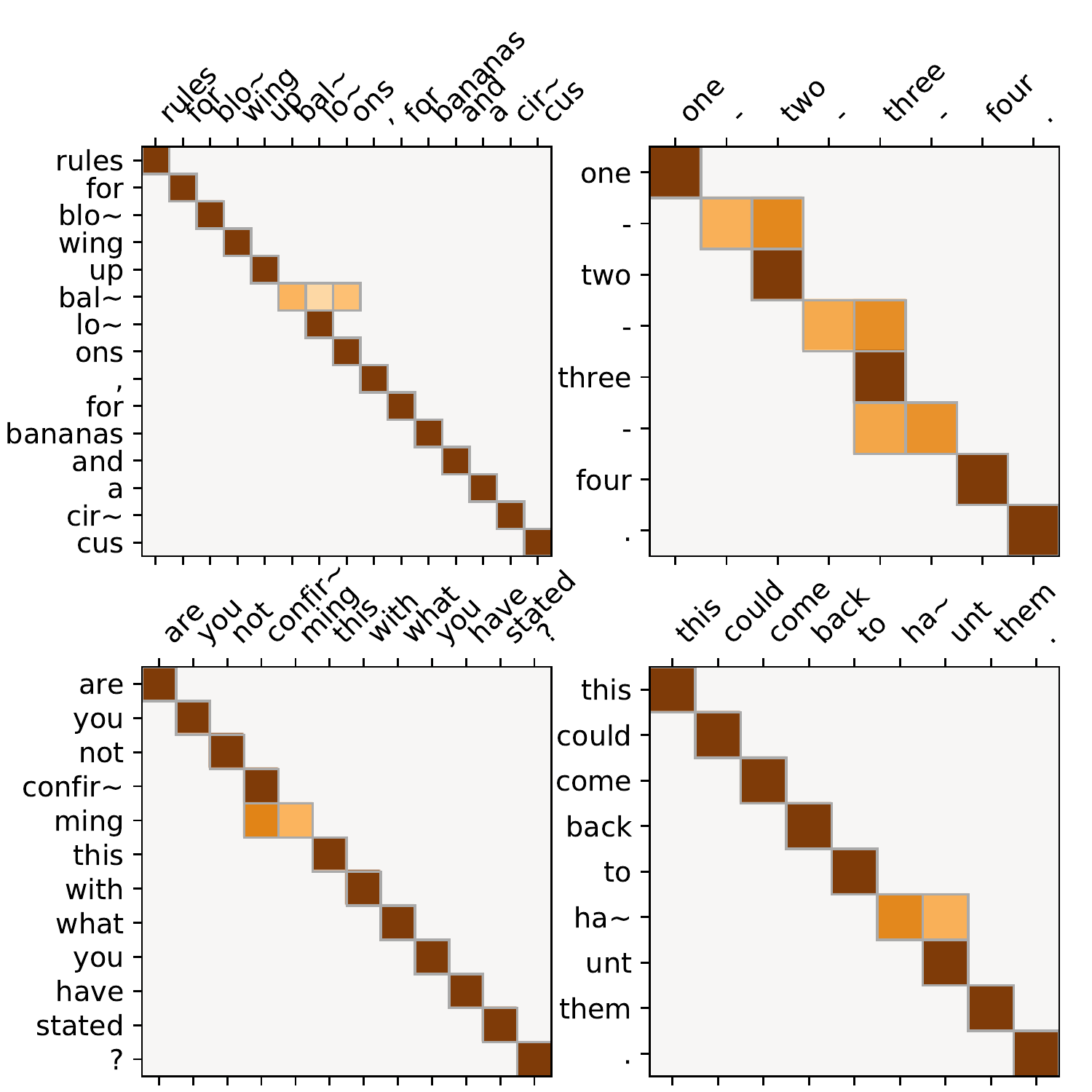}
    \caption{
    BPE-merging head $(\alpha=1.91)$ discovered in the
    $\alpha$-\entmaxtext model. Found in the first encoder layer,
    this head learns to discover some subword units and combine their
    information, leaving most words intact. It places $99.09\%$ of
    its probability mass within the same BPE cluster as the current
    token: more than any head in any other model.}
    \label{fig:head_bpe}
\end{figure}

\begin{figure}[t]
    \begin{subfigure}[b]{\columnwidth}
    \includegraphics[width=\columnwidth]{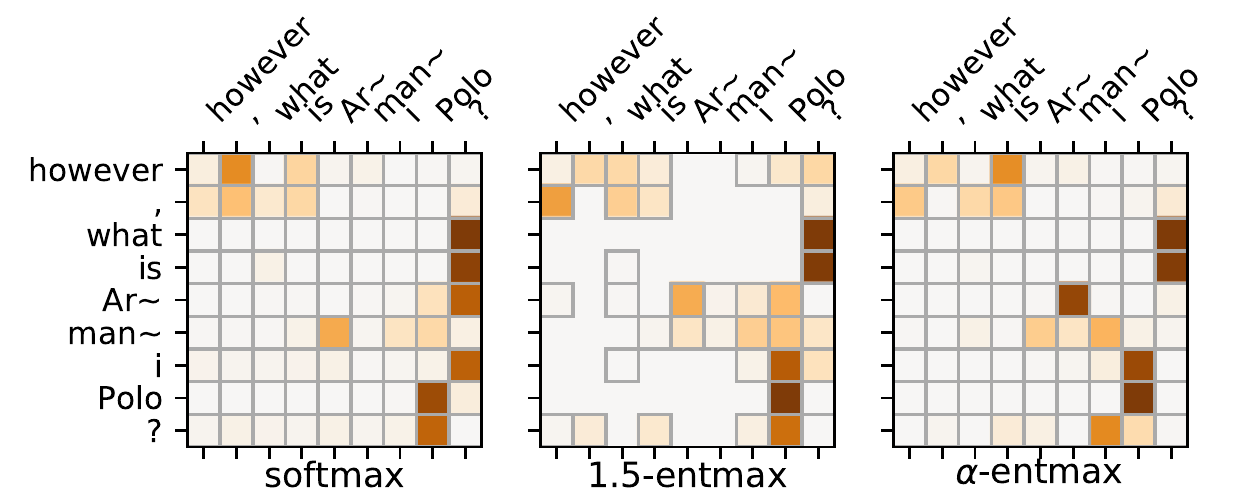}
    \end{subfigure}
    \begin{subfigure}[b]{\columnwidth}
        \includegraphics[width=\columnwidth]{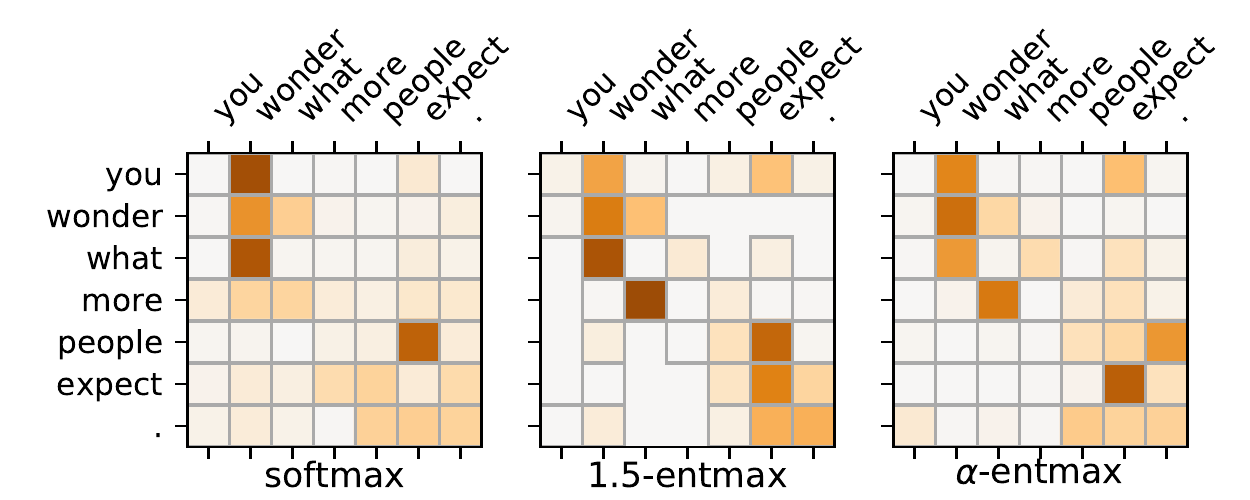}
    \end{subfigure}
\caption{
Interrogation-detecting heads in the three models. The top sentence
is interrogative while the bottom one is declarative but includes the
interrogative word ``what''. In the top example, these {\it
interrogation heads} assign a high probability to the question mark
in the time step of the interrogative word (with $\geq 97.0\%$
probability), while in the bottom example since there is no question
mark, the same head does not assign a high probability to the last
token in the sentence during the interrogative word time step.
Surprisingly, this head prefers a low $\alpha=1.05$, as can be seen
from the dense weights. This allows the head to identify the noun
phrase ``Armani Polo" better.}
\label{fig:head_interro}
\end{figure}

\begin{figure}[t]
    \includegraphics[
        width=\columnwidth]{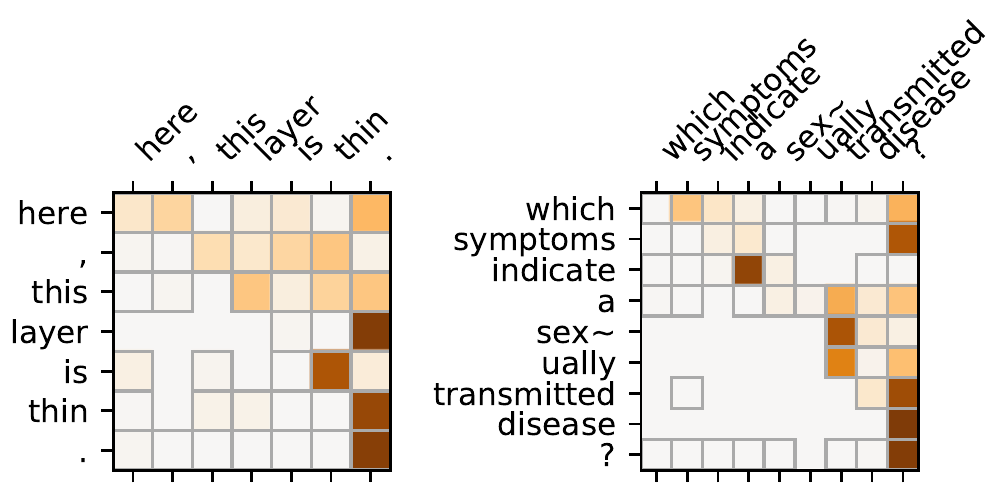}
\caption{
Example of two sentences of similar length where the same head
($\alpha=1.33$) exhibits different sparsity. The longer phrase in the
example on the right ``a sexually transmitted disease'' is handled
with higher confidence, leading to more sparsity.
}
\label{fig:sparsity_difference}
\end{figure}

\paragraph{Interrogation head.}
On the other hand, in \figref{head_interro} we show a head for which our
adaptively sparse model chose an $\alpha$ close to 1, making it
closer to softmax (also shown in {\it encoder, layer 1, head 3}
depicted in \figref{learning_alpha}). We observe that this head
assigns a high probability to question marks at the end of the
sentence in time steps where the current token is interrogative, thus
making it an interrogation-detecting head. We also observe this type
of heads in the other models, which we also depict in
\figref{head_interro}. The average attention weight placed on the
question mark when the current token is an interrogative word is
$98.5\%$ for softmax, $97.0\%$ for $1.5$-\entmaxtext, and $99.5\%$
for $\alpha$-\entmaxtext.

Furthermore, we can examine sentences where some tendentially sparse
heads become less so, thus identifying sources of ambiguity where the
head is less confident in its prediction. An example is shown in
\figref{sparsity_difference} where sparsity in the same head differs
for sentences of similar length.

\section{Related Work}\label{sec:related}

\paragraph{Sparse attention.}
Prior work has developed sparse attention mechanisms, including
applications to NMT~\citep{sparsemax, malaviya2018sparse, fusedmax,
shao2019ssn, maruf2019selective}. \citet{entmax} introduced the
\entmaxtext function this work builds upon. In their work, there is a
single attention mechanism which is controlled by a fixed $\alpha$.
In contrast, this is the first work to allow such attention mappings
to \emph{dynamically} adapt their curvature and sparsity, by
automatically adjusting the continuous $\alpha$ parameter. We also
provide the first results using sparse attention in a Transformer
model.

\paragraph{Fixed sparsity patterns.}
Recent research improves the scalability of Transformer-like networks
through static, fixed sparsity patterns
\citep{openai_sparse_transf,dynamic_conv}. Our adaptively-sparse
Transformer can dynamically select a sparsity pattern that finds
relevant words regardless of their position (\eg,
\figref{head_interro}). Moreover, the two strategies could be
combined. In a concurrent line of research, \citet{Sukhbaatar2019}
propose an adaptive attention span for Transformer language models.
While their work has each head learn a different contiguous span of
context tokens to attend to, our work finds different sparsity
patterns in the same span. Interestingly, some of their findings
mirror ours -- we found that attention heads in the last layers tend
to be denser on average when compared to the ones in the first
layers, while their work has found that lower layers tend to have a
shorter attention span compared to higher layers.

\paragraph{Transformer interpretability.}
The original Transformer paper~\citep{vaswani2017attention} shows
attention visualizations, from which some speculation can be made of
the roles the several attention heads have.
\citet{marecek-rosa-2018-extracting} study the syntactic abilities of
the Transformer self-attention, while \citet{raganato2018analysis}
extract dependency relations from the attention weights.
\citet{bert-rediscovers} find that the self-attentions in
BERT~\citep{devlin2018bert} follow a sequence of processes that
resembles a classical NLP pipeline. Regarding redundancy of heads,
\citet{specialized} develop a method that is able to prune heads of
the multi-head attention module and make an empirical study of the
role that each head has in self-attention (positional, syntactic and
rare words). \citet{li2018multi} also aim to reduce head redundancy
by adding a regularization term to the loss that maximizes head
disagreement and obtain improved results. While not considering
Transformer attentions, \citet{jain2019attention} show that
traditional attention mechanisms do not necessarily improve
interpretability since softmax attention is vulnerable to an
adversarial attack leading to wildly different model predictions for
the same attention weights. Sparse attention may mitigate these
issues; however, our work focuses mostly on a more mechanical aspect
of interpretation by analyzing head behavior, rather than on
explanations for predictions.

\section{Conclusion and Future Work}
We contribute a novel strategy for adaptively sparse attention, and,
in particular, for adaptively sparse Transformers. We present the
first empirical analysis of Transformers with sparse attention
mappings (\ie, \entmaxtext), showing potential in both translation
accuracy as well as in model interpretability.

In particular, we analyzed how the attention heads in the proposed
adaptively sparse Transformer can specialize more and with higher
confidence. Our adaptivity strategy relies only on gradient-based
optimization, side-stepping costly per-head hyper-parameter searches.
Further speed-ups are possible by leveraging more parallelism in the
bisection algorithm for computing $\alpha$-\entmaxtext.

Finally, some of the automatically-learned behaviors of our
adaptively sparse Transformers -- for instance, the near-deterministic
positional heads or the subword joining head -- may provide new ideas
for designing static variations of the Transformer.

\section*{Acknowledgments}
This work was supported by the European Research Council (ERC StG
DeepSPIN 758969), and by the Funda\c{c}\~ao para a Ci\^encia e
Tecnologia through contracts UID/EEA/50008/2019 and
CMUPERI/TIC/0046/2014 (GoLocal). We are grateful to Ben Peters for
the $\alpha$-\entmaxtext code and Erick Fonseca, Marcos Treviso,
Pedro Martins, and Tsvetomila Mihaylova for insightful group
discussion. We thank Mathieu Blondel for the idea to learn $\alpha$.
We would also like to thank the anonymous reviewers for their helpful
feedback.

\bibliography{refs}
\bibliographystyle{acl_natbib}

\clearpage
\onecolumn
\appendix
\begin{center}
{\huge \textbf{Supplementary Material}}
\end{center}

\setlength{\parindent}{0pt}
\setlength{\parskip}{1.5ex plus 0.5ex minus .2ex}

\def\RR{{\mathbb{R}}}
\def\EE{{\mathbb{E}}}
\def\RRY{\RR^{|\cY|}}
\def\y{\bm{y}}
\def\triangleY{\triangle^{|\cY|}}
\def\sizeY{{|\cY|}}
\def\cC{{\mathcal{C}}}
\def\cD{{\mathcal{D}}}
\def\cX{{\mathcal{X}}}
\def\cY{{\mathcal{Y}}}

\section{High-Level Statistics Analysis of Other Language Pairs}
\label{sec:plots_lps}

\begin{figure}[h!]
    \centering
    \begin{subfigure}[b]{.47\linewidth}
    \includegraphics[width=\linewidth]{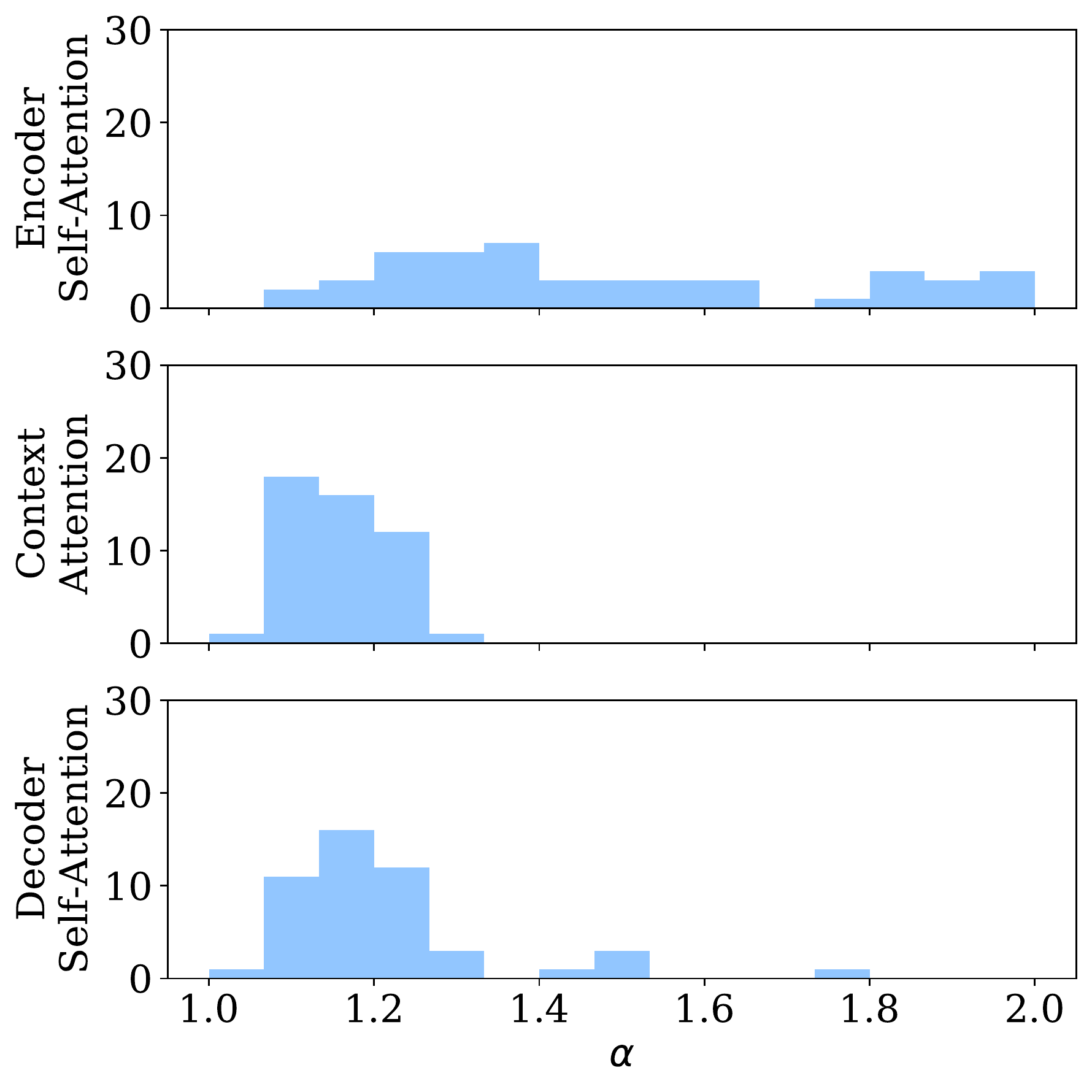}
    \caption{%
\label{fig:hist_alphas_ro}%
WMT 2016 \langp{ro}{en}.}
\end{subfigure}
\begin{subfigure}[b]{.47\linewidth}
    \includegraphics[width=\linewidth]{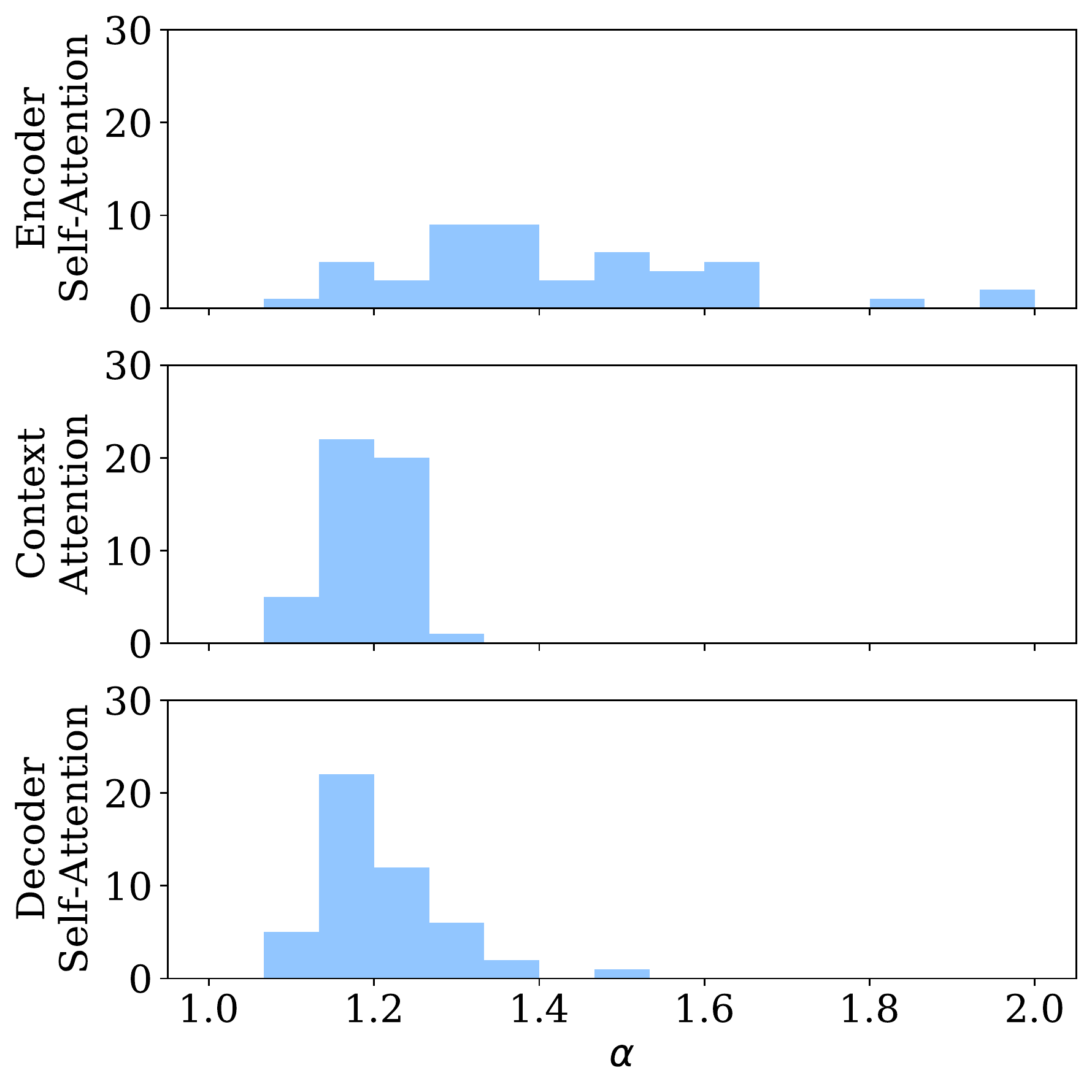}
    \caption{%
\label{fig:hist_alphas_ja}%
KFTT \langp{ja}{en}.}
\end{subfigure}

\begin{subfigure}[b]{.47\linewidth}
    \includegraphics[width=\linewidth]{figures/hist_alphas.pdf}
    \caption{%
\label{fig:hist_alphas_en}%
WMT 2014 \langp{en}{de}.}
\end{subfigure}
\begin{subfigure}[b]{.47\linewidth}
    \includegraphics[width=\linewidth]{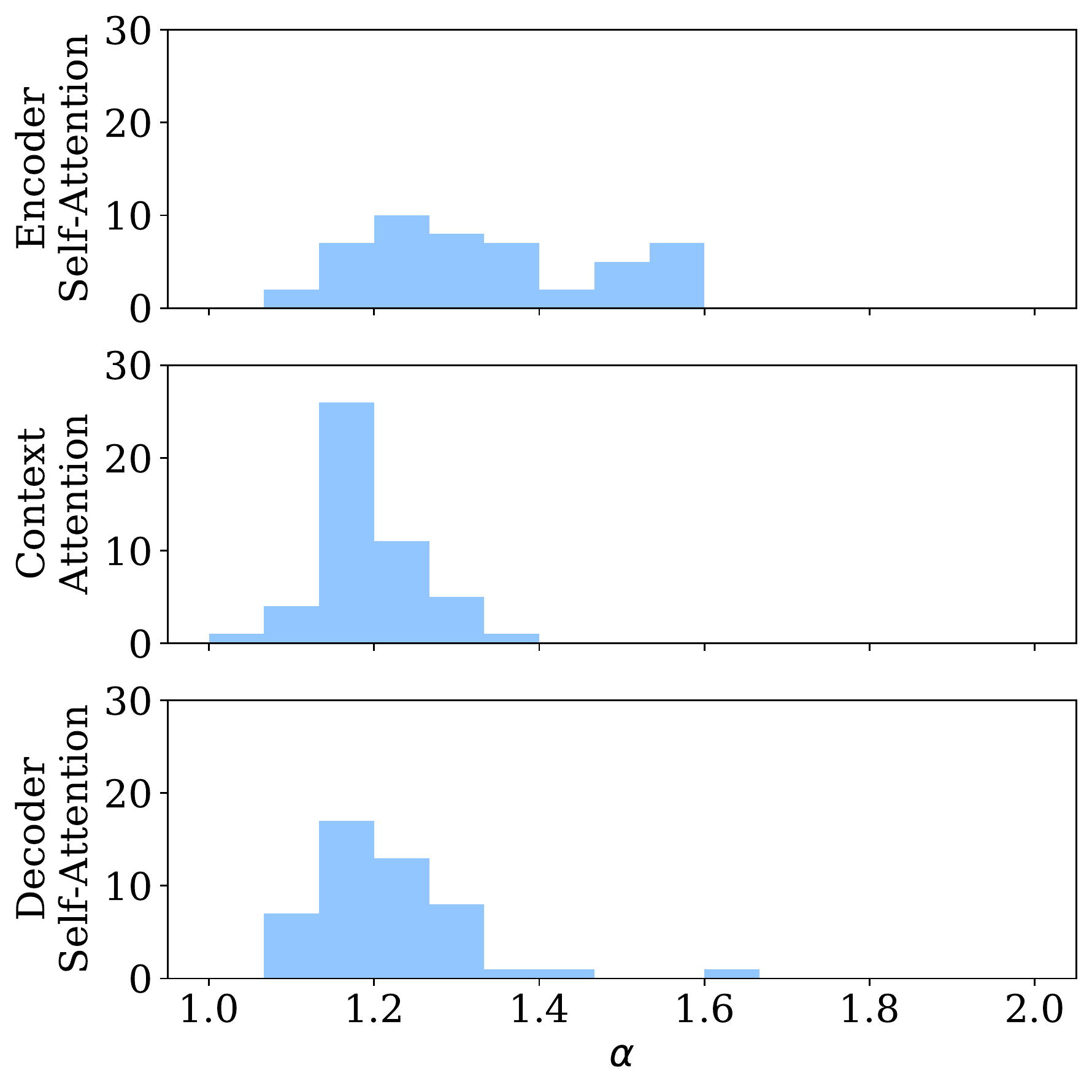}
    \caption{%
\label{fig:hist_alphas_de}%
IWSLT 2017 \langp{de}{en}.}
\end{subfigure}
\caption{Histograms of $\alpha$ values.}
\label{fig:hist_alphas_lps}
\end{figure}

\begin{figure}[h!]
    \centering
    \begin{subfigure}[b]{.47\linewidth}
    \includegraphics[width=\linewidth]{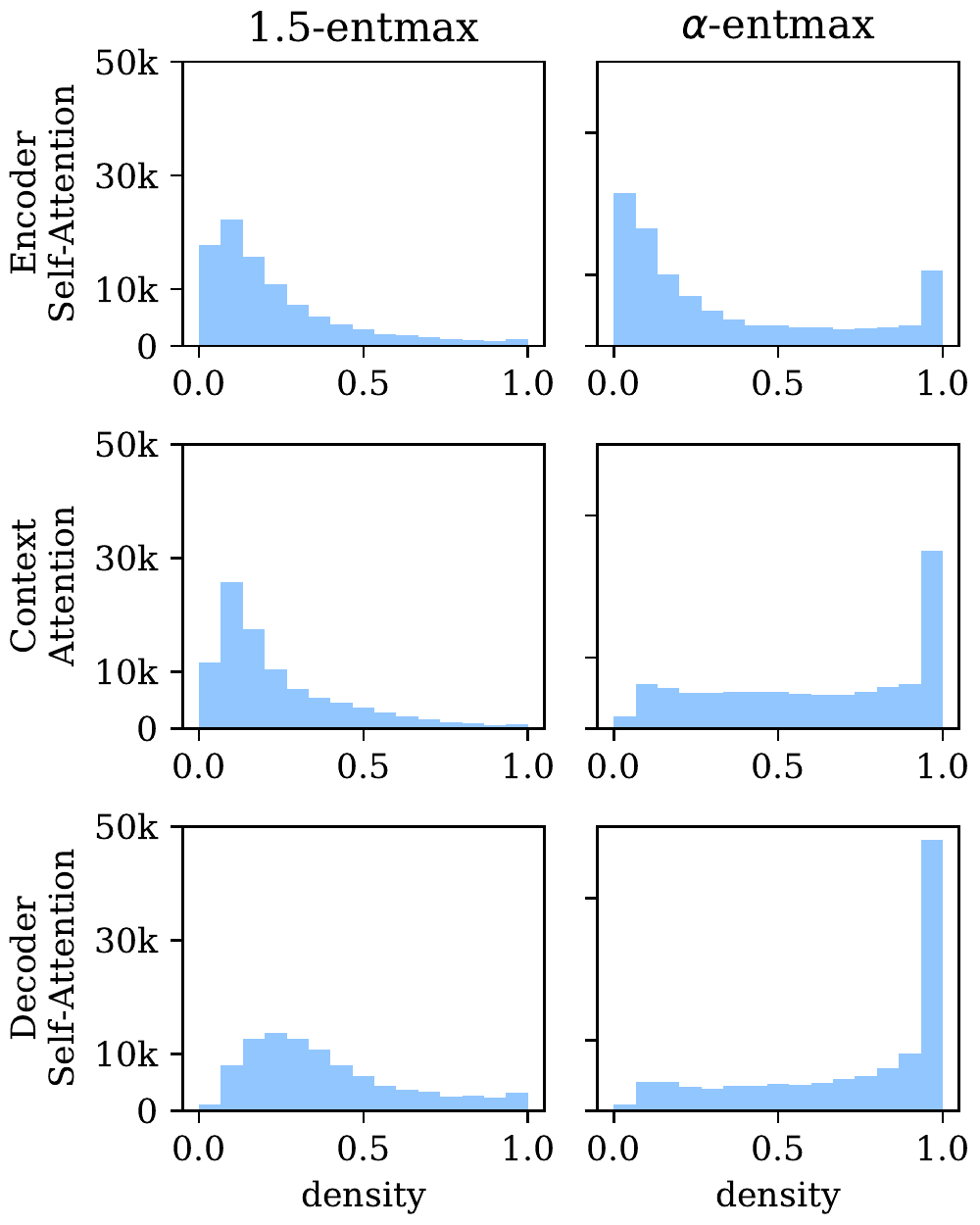}
    \caption{%
\label{fig:hist_densities_ro}%
WMT 2016 \langp{ro}{en}.}
\end{subfigure}
\begin{subfigure}[b]{.47\linewidth}
    \includegraphics[width=\linewidth]{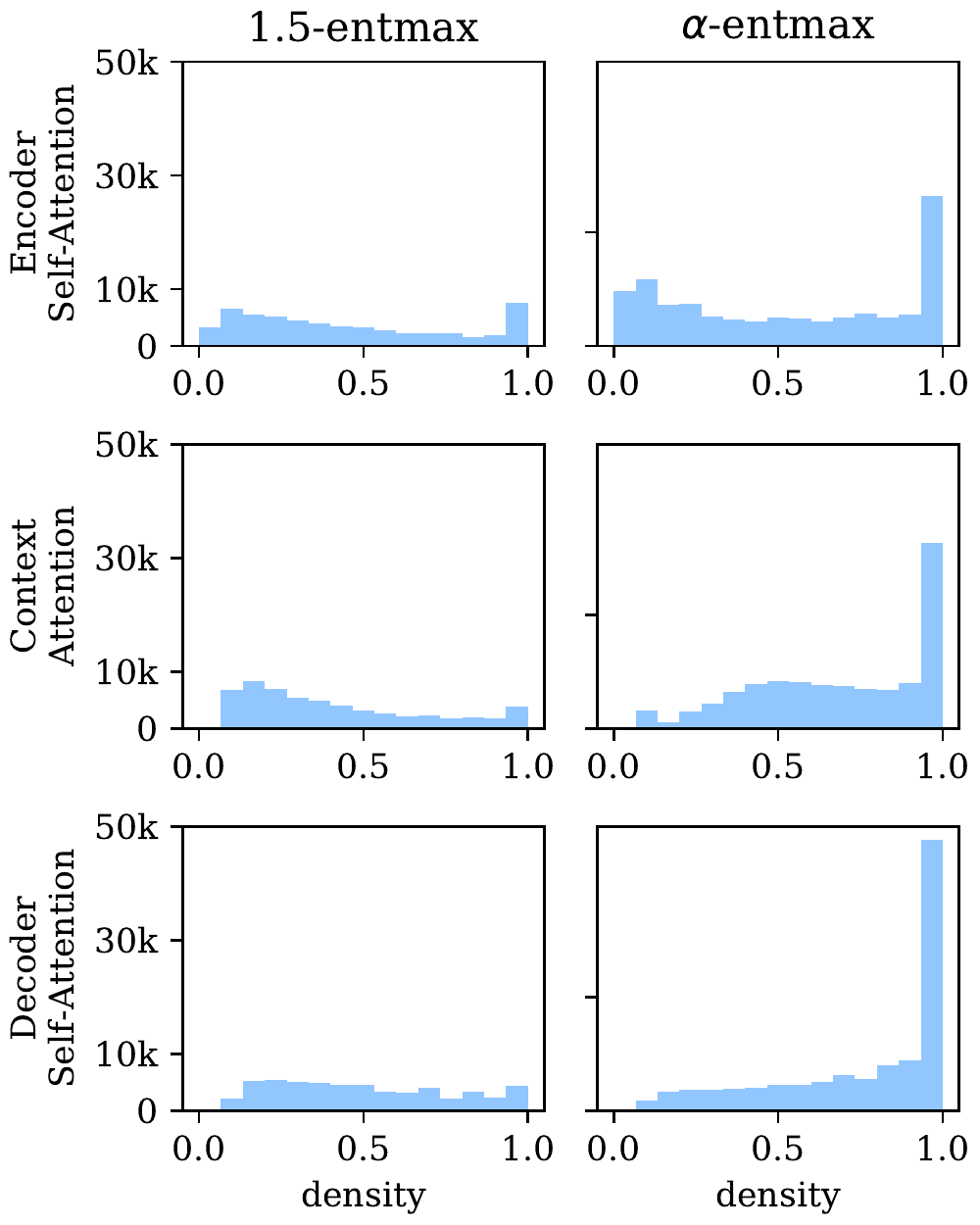}
    \caption{%
\label{fig:hist_densities_ja}%
KFTT \langp{ja}{en}.}
\end{subfigure}

\begin{subfigure}[b]{.47\linewidth}
    \includegraphics[width=\linewidth]{figures/hist_densities.pdf}
    \caption{%
\label{fig:hist_densities_en}%
WMT 2014 \langp{en}{de}.}
\end{subfigure}
\begin{subfigure}[b]{.47\linewidth}
    \includegraphics[width=\linewidth]{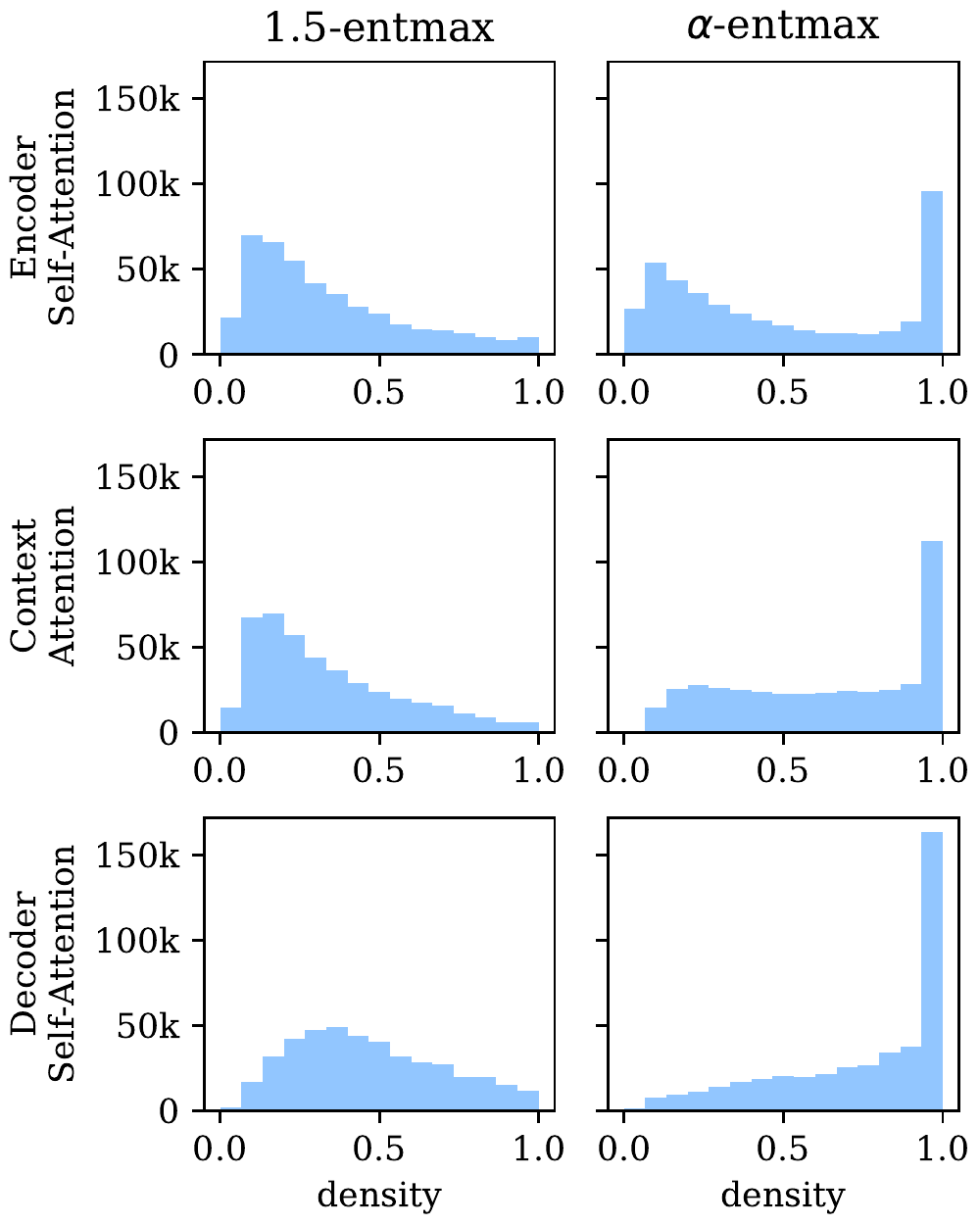}
    \caption{%
\label{fig:hist_densities_de}%
IWSLT 2017 \langp{de}{en}.}
\end{subfigure}
\caption{Histograms of head densities.}
\label{fig:hist_densities_lps}
\end{figure}

\begin{figure}[h!]
    \centering
    \begin{subfigure}[b]{.47\linewidth}
    \includegraphics[width=\linewidth]{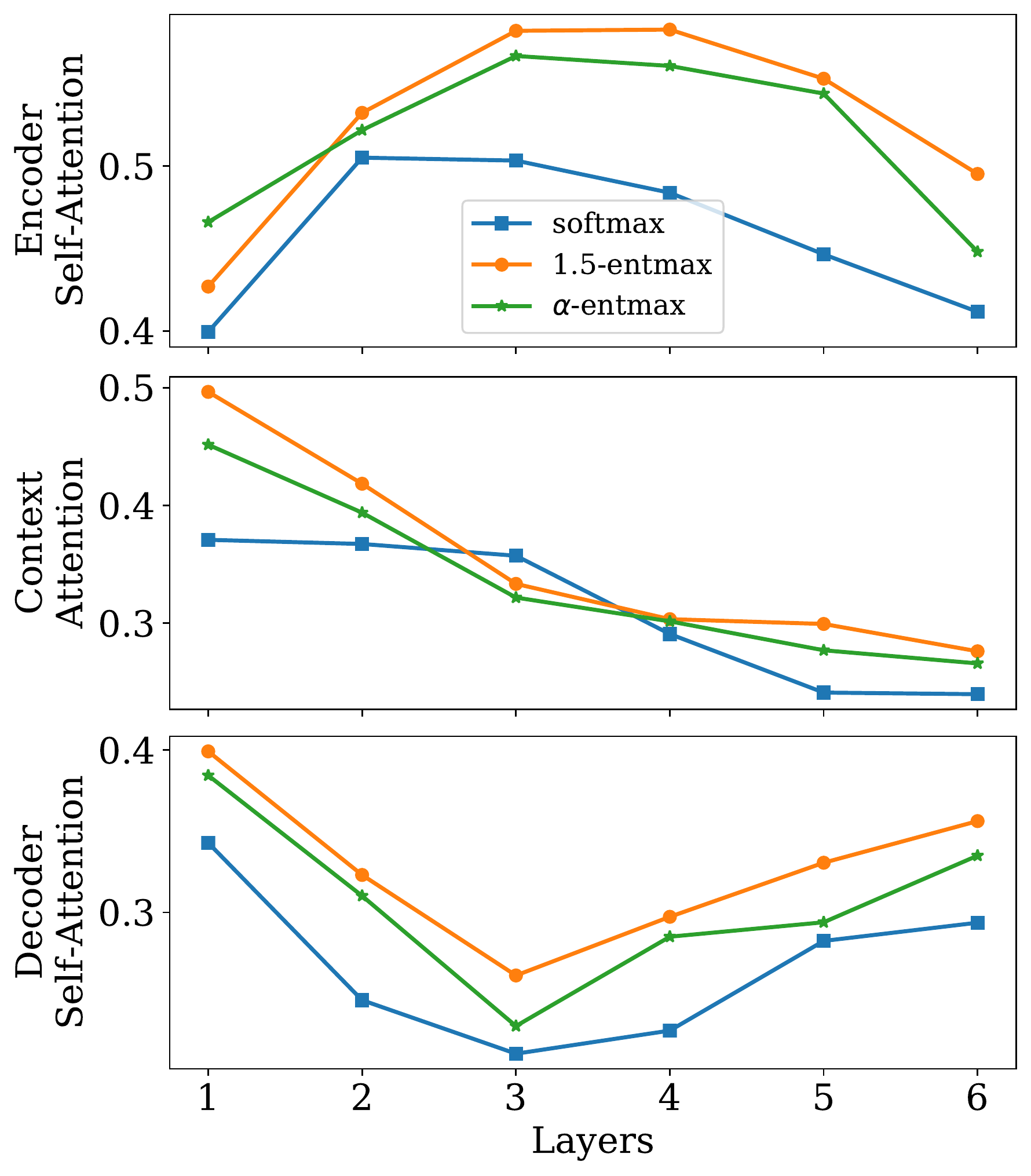}
    \caption{%
\label{fig:js_divs_ro}%
WMT 2016 \langp{ro}{en}.}
\end{subfigure}
\begin{subfigure}[b]{.47\linewidth}
    \includegraphics[width=\linewidth]{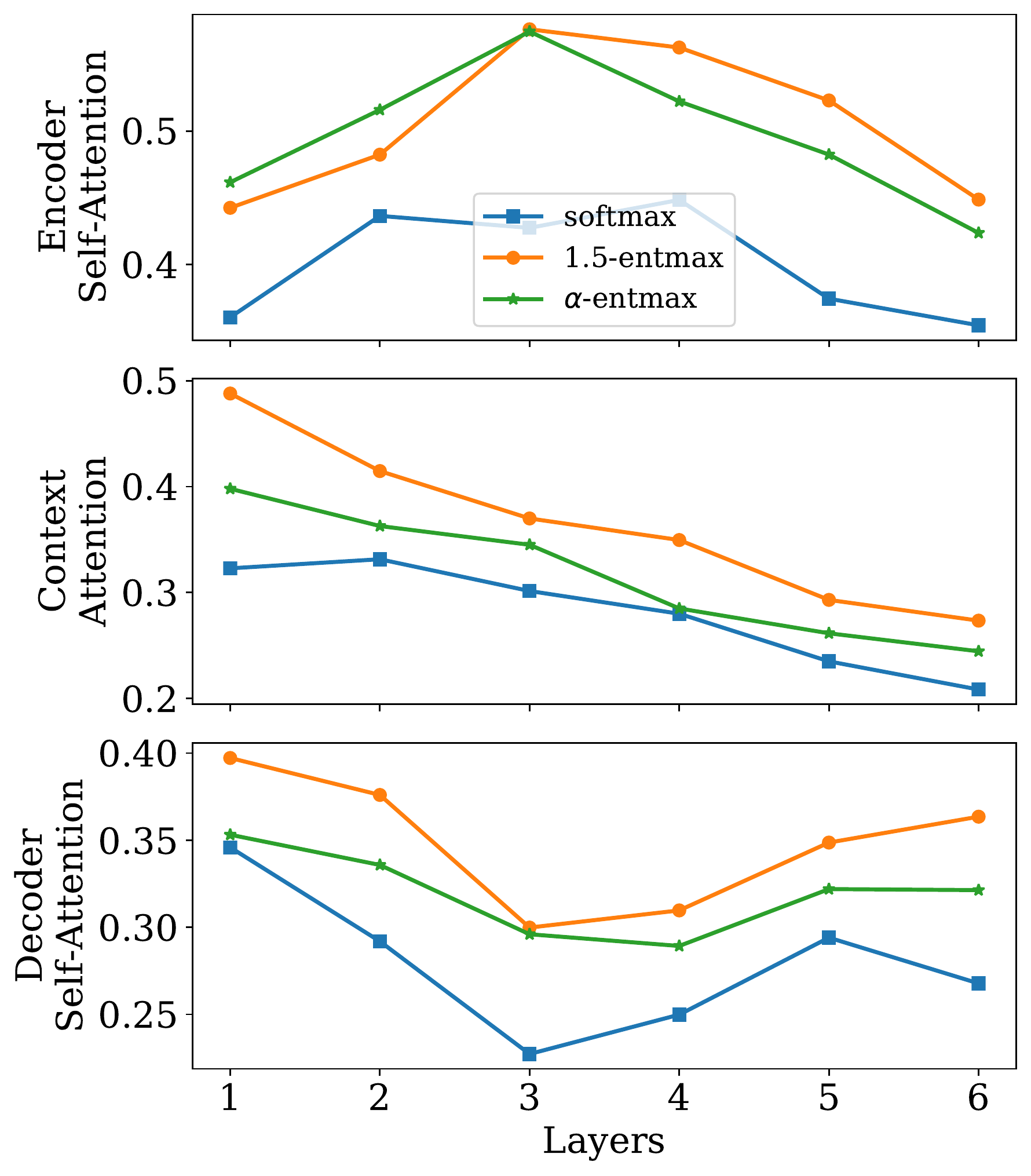}
    \caption{%
\label{fig:js_divs_ja}%
KFTT \langp{ja}{en}.}
\end{subfigure}

\begin{subfigure}[b]{.47\linewidth}
    \includegraphics[width=\linewidth]{figures/js_divs.pdf}
    \caption{%
\label{fig:js_divs_en}%
WMT 2014 \langp{en}{de}.}
\end{subfigure}
\begin{subfigure}[b]{.47\linewidth}
    \includegraphics[width=\linewidth]{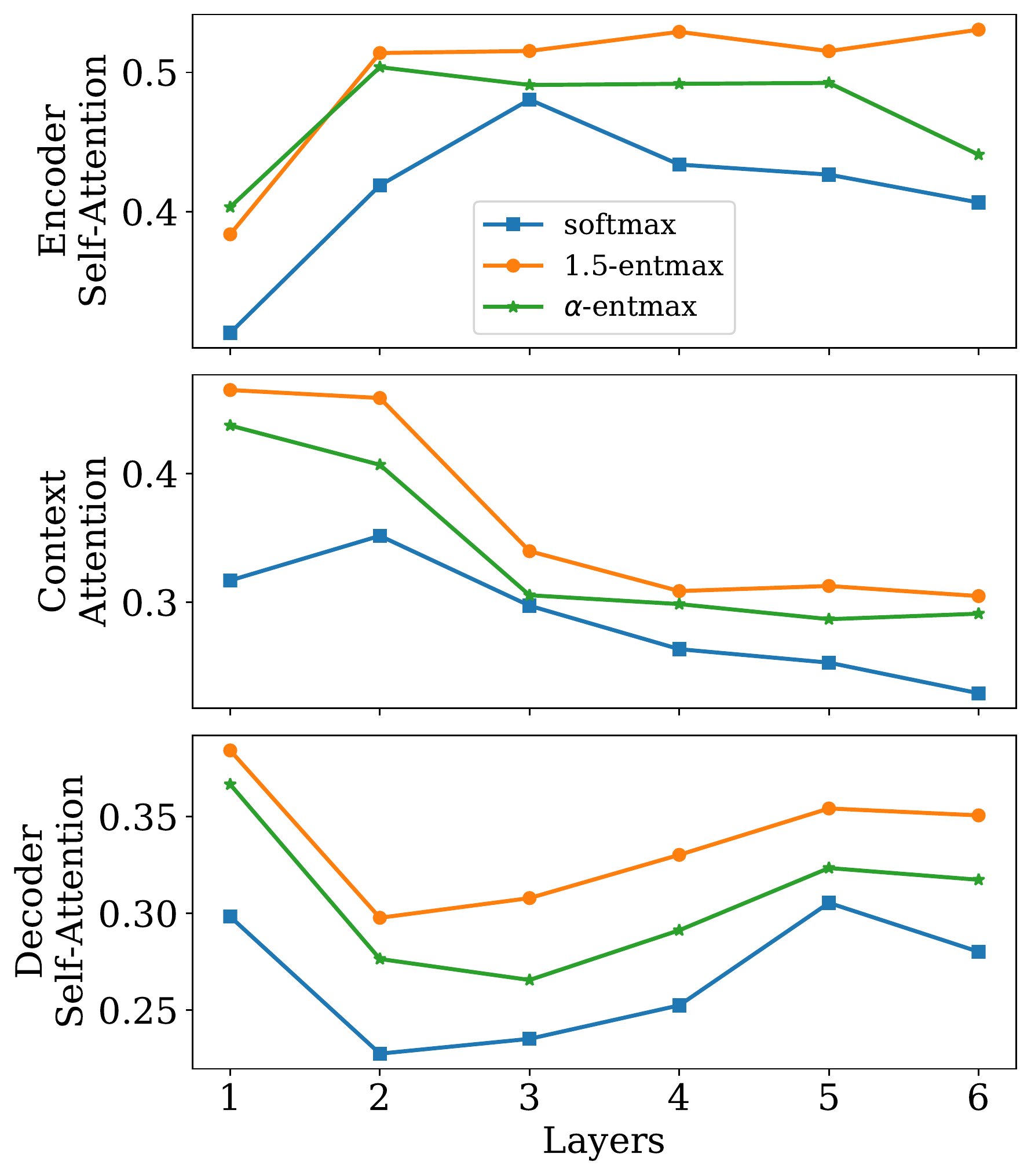}
    \caption{%
\label{fig:js_divs_de}%
IWSLT 2017 \langp{de}{en}.}
\end{subfigure}
\caption{Jensen-Shannon divergence over layers.}
\label{fig:js_divs_lps}
\end{figure}

\begin{figure}[h!]
    \centering
    \begin{subfigure}[b]{.47\linewidth}
    \includegraphics[width=\linewidth]{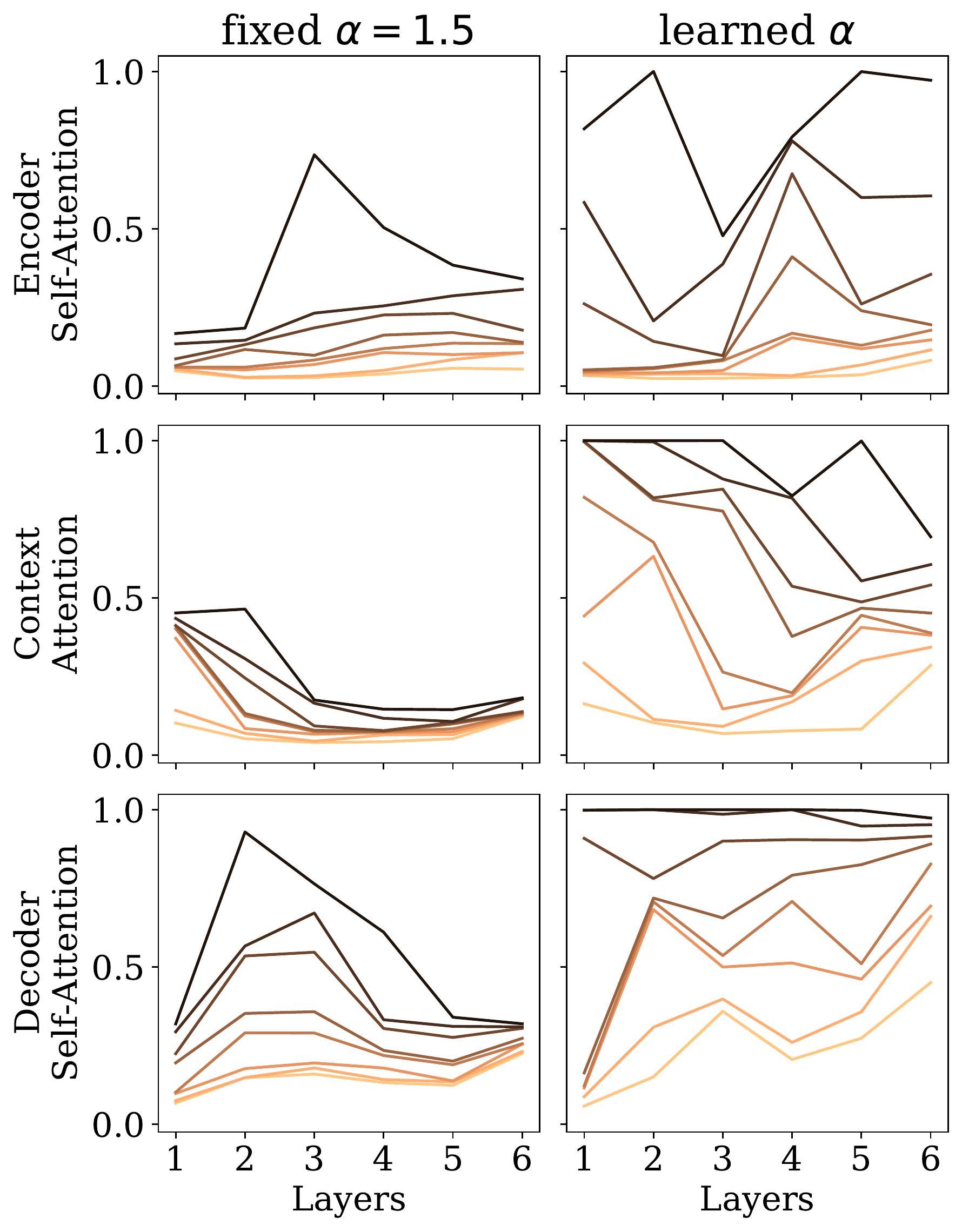}
    \caption{%
\label{fig:head_density_per_layer_ro}%
WMT 2016 \langp{ro}{en}.}
\end{subfigure}
\begin{subfigure}[b]{.47\linewidth}
    \includegraphics[width=\linewidth]{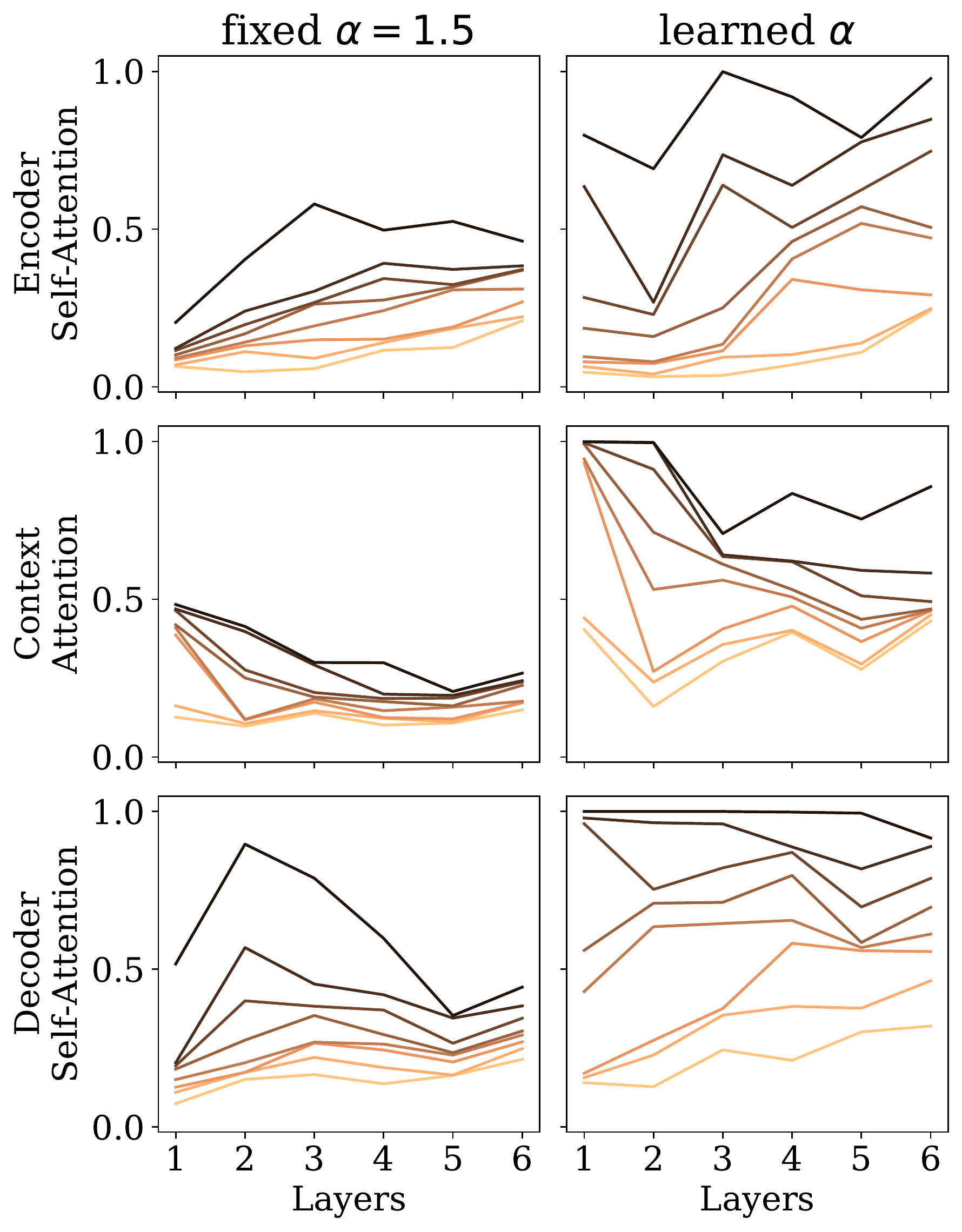}
    \caption{%
\label{fig:head_density_per_layer_ja}%
KFTT \langp{ja}{en}.}
\end{subfigure}

\begin{subfigure}[b]{.47\linewidth}
    \includegraphics[width=\linewidth]{figures/head_density_per_layer.pdf}
    \caption{%
\label{fig:head_density_per_layer_en}%
WMT 2014 \langp{en}{de}.}
\end{subfigure}
\begin{subfigure}[b]{.47\linewidth}
    \includegraphics[width=\linewidth]{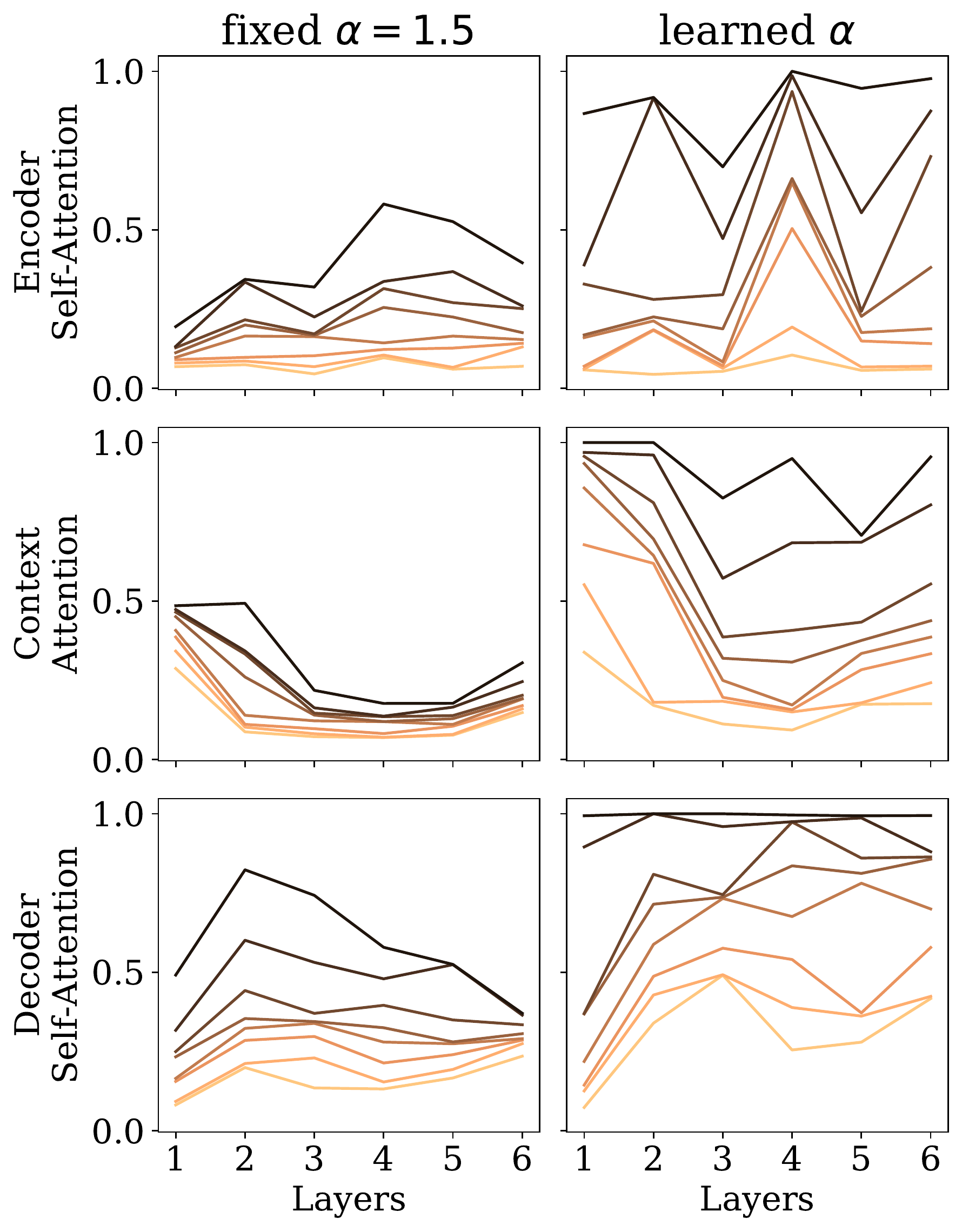}
    \caption{%
\label{fig:head_density_per_layer_de}%
IWSLT 2017 \langp{de}{en}.}
\end{subfigure}
\caption{Head densities over layers.}
\label{fig:head_density_per_layer_lps}
\end{figure}

\clearpage

\section{Background}

\subsection{Regularized Fenchel-Young prediction functions}
\begin{definition}[\citealt{blondel2019learning}]\label{def:rpf}

Let $\Omega \colon \simplex^d \to \RR \cup \{\infty\}$ be a strictly convex
regularization function. We define the prediction function $\amap_{\Omega}$ as
\begin{equation}
\amap_{\Omega}(\x) = \argmax_{\p \in \simplex^d}
    \big(\p^\top\x - \Omega(\p)\big)
\label{eq:amap}
\end{equation}
\end{definition}

\subsection{Characterizing the {\boldmath $\alpha$-\entmaxtext} mapping}\label{sec:bgform}

\begin{lemma}[\citealt{entmax}]
\label{lemma:tsallis_reduction}%
For any $\x$, there exists a unique $\tau^\star$ such that
\begin{equation}\label{eq:entmax_form_supp}
\aentmax(\x)
= [(\alpha - 1){\x} - \tau^\star \ones]_+^{\nicefrac{1}{\alpha-1}}.
\end{equation}
\end{lemma}
\begin{proof}
From the definition of $\aentmax$,
\begin{equation}
    \aentmax(\bm{z}) \coloneqq
    \argmax_{\p \in \simplex^d} \bm{p}^\top\bm{z} + \HHt_{\alpha}(\bm{p}),
\end{equation}
we may easily identify it with a regularized prediction function
(Def.~\ref{def:rpf}):
\[\aentmax(\x) \equiv \amap_{-\HHt_{\alpha}}(\x).\]
We first note that for all $\p \in \triangle^d$,
\begin{equation}
-(\alpha-1) \HHt_\alpha(\p) = \frac{1}{\alpha} \sum_{i=1}^d p_i^\alpha + \text{const}.
\end{equation}
From the constant invariance and
scaling properties of $\amap_{\Omega}$
\citep[Proposition~1, items~4--5]{blondel2019learning},
\[\amap_{-\HHt_\alpha}(\x)
= \amap_{\Omega}((\alpha - 1)\x),
\quad\text{with}\quad
\Omega(\p) =
\sum_{j=1}^{d} g(\pp_j),
\quad
g(t) = \frac{t^\alpha}{\alpha}.
\]
Using \citep[Proposition~5]{blondel2019learning}, noting that
$g'(t) = t^{\alpha - 1}$ and $(g')^{-1}(u) =  u^{\nicefrac{1}{\alpha-1}}$,
yields
\begin{equation}\label{eq:rootform}
\amap_{\Omega}(\x) = [\x - \tau^\star \ones]_+^{\nicefrac{1}{\alpha-1}},
\quad\text{and therefore}\quad
\aentmax(\x) = [(\alpha-1)\x - \tau^\star \ones]_+^{\nicefrac{1}{\alpha-1}}.
\end{equation}
Since $\HHt_\alpha$ is strictly convex on the simplex, $\aentmax$ has a unique
solution $\p^\star$. \eqnref{rootform}
implicitly defines a one-to-one mapping between $\p^\star$ and $\tau^\star$
as long as $\p^\star \in \simplex$,
therefore $\tau^\star$ is also unique.
\end{proof}

\subsection{Connections to softmax and sparsemax}\label{sec:softmax}
The Euclidean projection onto the simplex, sometimes referred to, in the context of
neural attention, as sparsemax \citep{sparsemax},
is defined as
\begin{equation}
\sparsemax(\x) \coloneqq \argmin_{\p \in \simplex} \| \p - \x \|_2^2.
\end{equation}
The solution can be characterized through the
unique threshold $\thresh$ such that $\sum_i
\sparsemax(\x)_i = 1$ and \citep{Held1974}
\begin{equation}
\sparsemax(\x) = [\x - \tau \bm{1}]_+~.
\end{equation}
Thus, each coordinate of the sparsemax solution is a piecewise-linear function.
Visibly, this expression is recovered when setting $\alpha=2$ in the
$\alpha$-\entmaxtext expression (\eqnref{entmax_form_supp}); for other
values of $\alpha$, the exponent induces curvature.

On the other hand, the well-known softmax is usually defined through the
expression
\begin{equation}\label{eq:softmax-supp}
\softmax(\x)_i \coloneqq \frac{\exp(\xx_i)}{\sum_j \exp(\xx_j)},
\end{equation}
which can be shown to be the unique solution of the optimization problem
\begin{equation}
\softmax(\x)_i =
\argmax_{\p \in \simplex} \p^\top\x + \HHs(\p),
\end{equation}
where $\HHs(\p) \coloneqq -\sum_i \pp_i \log \pp_i$ is the Shannon entropy.
Indeed, setting the gradient to $0$ yields the condition
$\log \pp_i = \xx_j - \nu_i - \tau - 1$, where $\thresh$ and $\nu > 0$ are Lagrange
multipliers for the simplex constraints $\sum_i \pp_i = 1$ and $\pp_i \geq 0$,
respectively. Since the \lhs is only finite for $\pp_i>0$,
we must have $\nu_i=0$ for all $i$, by complementary
slackness. Thus, the solution must have the form $\pp_i =
\nicefrac{\exp(\xx_i)}{Z}$, yielding \eqnref{softmax-supp}.
\section{Jacobian of {\boldmath $\alpha$}-\entmaxtext \wrt the shape parameter
{\boldmath $\alpha$}:
Proof of Proposition~\ref{prop:grad_alpha}}
\label{sec:alpha_grad}
Recall that the \entmaxtext transformation is defined as:
\begin{equation}\label{eq:entmax_form_supp}
    \aentmax(\bm{z}) \coloneqq
    \argmax_{\p \in \simplex^d} \bm{p}^\top\bm{z} + \HHt_{\alpha}(\bm{p}),
\end{equation}
where $\alpha \geq 1$ and $\HHt_{\alpha}$ is the Tsallis entropy,
\begin{equation}%
    \HHt_{\alpha}(\bm{p})\!\coloneqq\!
\begin{cases}
\frac{1}{\alpha(\alpha-1)}\sum_j\!\left(p_j - p_j^\alpha\right)\!, &
\!\!\!\alpha \ne 1,\\
\HHs(\bm{p}), &
\!\!\!\alpha = 1,
\end{cases}
\end{equation}
and $\HHs(\bm{p}):= -\sum_j p_j \log p_j$ is the Shannon entropy.

In this section, we derive the Jacobian of $\entmax$ with respect to the scalar parameter $\alpha$.

\subsection{General case of {\boldmath $\alpha>1$}}

From the KKT conditions associated with the optimization problem in
Eq.~\ref{eq:entmax_form_supp}, we have that the solution $\bm{p}^{\star}$ has the following form, coordinate-wise:
\begin{equation}\label{eq:p_kkt}
    p_i^{\star} = [(\alpha-1)(z_i - \tau^{\star})]_+^{1/(\alpha-1)},
\end{equation}
where $\tau^{\star}$ is a scalar Lagrange multiplier that ensures that
$\bm{p}^{\star}$ normalizes to 1, \ie, it is defined implicitly by the condition:
\begin{equation}\label{eq:tau_condition}
    \sum_i [(\alpha-1)(z_i - \tau^{\star})]_+^{1/(\alpha-1)} = 1.
\end{equation}
For general values of $\alpha$, Eq.~\ref{eq:tau_condition} lacks a closed form solution. This makes the computation of the
Jacobian
\begin{equation}
    \frac{\partial \aentmax(\bm{z})}{\partial \alpha}
\end{equation}
non-trivial. Fortunately, we can use the technique of implicit differentiation
to obtain this Jacobian.

The Jacobian exists almost everywhere, and the expressions we derive
expressions yield a generalized Jacobian \citep{clarke_book} at any
non-differentiable points that may occur for certain ($\alpha$, $\x$) pairs.
We begin by noting that $\frac{\partial p_i^{\star}}{\partial \alpha} = 0$ if
$p_i^{\star} = 0$, because increasing $\alpha$ keeps sparse coordinates
sparse.\footnote{This follows from the margin property of $\HHt_\alpha$
 \citep{blondel2019learning}.}
Therefore we need to worry only
about coordinates that are in the support of $\bm{p}^\star$. We will assume
hereafter that the $i$\textsuperscript{th} coordinate of $\bm{p}^\star$ is non-zero.
We have:
\begin{eqnarray}\label{eq:gradient_alpha_01}
    \frac{\partial p_i^{\star}}{\partial \alpha} &=& \frac{\partial}{\partial \alpha} [(\alpha-1)(z_i - \tau^{\star})]^{\frac{1}{\alpha-1}}\nonumber\\
    &=& \frac{\partial}{\partial \alpha} \exp \left[\frac{1}{\alpha-1} \log [(\alpha-1)(z_i - \tau^{\star})]\right]\nonumber\\
    &=& p_i^{\star} \frac{\partial}{\partial \alpha} \left[\frac{1}{\alpha-1} \log [(\alpha-1)(z_i - \tau^{\star})]\right]\nonumber\\
    &=& \frac{p_i^{\star}}{(\alpha-1)^2} \left[\frac{\frac{\partial}{\partial \alpha} [(\alpha-1)(z_i - \tau^{\star})]}{z_i - \tau^{\star}} - \log[(\alpha-1)(z_i - \tau^{\star})] \right]\nonumber\\
    &=& \frac{p_i^{\star}}{(\alpha-1)^2} \left[\frac{z_i - \tau^{\star} - (\alpha-1)\frac{\partial \tau^{\star}}{\partial \alpha} }{z_i - \tau^{\star}} - \log[(\alpha-1)(z_i - \tau^{\star})] \right]\nonumber\\
    &=& \frac{p_i^{\star}}{(\alpha-1)^2} \left[1 - \frac{\alpha-1}{z_i - \tau^{\star}}\frac{\partial \tau^{\star}}{\partial \alpha} - \log[(\alpha-1)(z_i - \tau^{\star})] \right].
\end{eqnarray}
We can see that this Jacobian depends on $\frac{\partial \tau^{\star}}{\partial \alpha}$, which we now compute using implicit differentiation.

Let $\mathcal{S} = \{i: \pp^\star_i > 0 \}$).
By differentiating both sides of Eq.~\ref{eq:tau_condition}, re-using some of
the steps in Eq.~\ref{eq:gradient_alpha_01}, and recalling Eq.~\ref{eq:p_kkt},
we get
\begin{eqnarray}\label{eq:gradient_tau_implicit}
    0 &=& \sum_{i \in \mathcal{S}} \frac{\partial}{\partial \alpha} [(\alpha-1)(z_i - \tau^{\star})]^{1/(\alpha-1)}\nonumber\\
    &=& \sum_{i \in \mathcal{S}} \frac{p_i^{\star}}{(\alpha-1)^2} \left[1 - \frac{\alpha-1}{z_i - \tau^{\star}}\frac{\partial \tau^{\star}}{\partial \alpha} - \log[(\alpha-1)(z_i - \tau^{\star})] \right]\nonumber\\
    &=&  \frac{1}{(\alpha-1)^2} - \frac{\partial \tau^{\star}}{\partial \alpha} \sum_{i \in \mathcal{S}} \frac{p_i^{\star}}{(\alpha - 1)(z_i - \tau^{\star})} - \sum_{i \in \mathcal{S}} \frac{p_i^{\star}}{(\alpha-1)^2} \log[(\alpha-1)(z_i - \tau^{\star})] \nonumber\\
    &=&  \frac{1}{(\alpha-1)^2} - \frac{\partial \tau^{\star}}{\partial \alpha} \sum_{i} (p_i^{\star})^{2-\alpha} - \sum_{i} \frac{p_i^{\star}}{\alpha-1} \log p_i^{\star}\nonumber\\
    &=&  \frac{1}{(\alpha-1)^2} - \frac{\partial \tau^{\star}}{\partial \alpha} \sum_{i} (p_i^{\star})^{2-\alpha} + \frac{\HHs(\bm{p}^*)}{\alpha-1},
\end{eqnarray}
from which we obtain:
\begin{eqnarray}\label{eq:gradient_tau}
    \frac{\partial \tau^{\star}}{\partial \alpha} &=& \frac{\frac{1}{(\alpha-1)^2} + \frac{\HHs(\bm{p}^{\star})}{\alpha-1}}{\sum_i (p_i^{\star})^{2-\alpha}}.
\end{eqnarray}

Finally, plugging Eq.~\ref{eq:gradient_tau} into Eq.~\ref{eq:gradient_alpha_01}, we get:
\begin{eqnarray}\label{eq:gradient_alpha}
    \frac{\partial p_i^{\star}}{\partial \alpha} &=&  \frac{p_i^{\star}}{(\alpha-1)^2} \left[1 - \frac{1}{(p_i^{\star})^{\alpha-1}}\frac{\partial \tau^{\star}}{\partial \alpha} - (\alpha-1)\log p_i^{\star} \right]\nonumber\\
    &=&  \frac{p_i^{\star}}{(\alpha-1)^2} \left[1 - \frac{1}{(p_i^{\star})^{\alpha-1}}\frac{\frac{1}{(\alpha-1)^2} + \frac{\HHs(\bm{p}^{\star})}{\alpha-1}}{\sum_i (p_i^{\star})^{2-\alpha}} - (\alpha-1)\log p_i^{\star} \right]\nonumber\\
    &=& \frac{p_i^{\star} - \tilde{p}_i(\alpha)}{(\alpha-1)^2} - \frac{p_i^{\star}\log p_i^{\star} + \tilde{p}_i(\alpha)\HHs(\bm{p}^{\star})}{\alpha-1},
\end{eqnarray}
where we denote by
\begin{equation}
    \tilde{p}_i(\alpha) = \frac{(p_i^{\star})^{2-\alpha}}{\sum_j (p_j^{\star})^{2-\alpha}}.
\end{equation}
The distribution $\tilde{\bm{p}}(\alpha)$ can be interpreted as a ``skewed''
distribution obtained from $\bm{p}^{\star}$, which appears in the Jacobian of
$\aentmax(\x)$ \wrt $\x$ as well~\cite{entmax}.

\subsection{Solving the indetermination for $\alpha=1$}

We can write Eq.~\ref{eq:gradient_alpha} as
\begin{eqnarray}\label{eq:gradient_alpha_fraction}
    \frac{\partial p_i^{\star}}{\partial \alpha} &=&
    \frac{p_i^{\star} - \tilde{p}_i(\alpha) - (\alpha-1)(p_i^{\star}\log p_i^{\star} + \tilde{p}_i(\alpha)\HHs(\bm{p}^{\star}))}{(\alpha-1)^2}.
\end{eqnarray}
When $\alpha \rightarrow 1^+$, we have $\tilde{\bm{p}}(\alpha) \rightarrow \bm{p}^{\star}$, which leads to a $\frac{0}{0}$ indetermination.

To solve this indetermination, we will need to apply L'H\^opital's rule twice.
Let us first compute the derivative of $\tilde{p}_i(\alpha)$ with respect to $\alpha$. We have
\begin{equation}
\frac{\partial}{\partial \alpha} (p_i^\star)^{2-\alpha} = -(p_i^{\star})^{2-\alpha} \log p_i^{\star},
\end{equation}
therefore
\begin{eqnarray}
\frac{\partial}{\partial \alpha} \tilde{p}_i(\alpha) &=& \frac{\partial}{\partial \alpha} \frac{(p_i^\star)^{2-\alpha}}{\sum_j (p_j^\star)^{2-\alpha}}\nonumber\\
&=& \frac{-(p_i^{\star})^{2-\alpha} \log p_i^{\star} \sum_j (p_j^\star)^{2-\alpha} + (p_i^{\star})^{2-\alpha} \sum_j (p_j^{\star})^{2-\alpha} \log p_j^{\star}}{\left( \sum_j (p_j^\star)^{2-\alpha} \right)^2}\nonumber\\
&=& -\tilde{p}_i(\alpha)\log p_i^{\star} + \tilde{p}_i(\alpha) \sum_j \tilde{p}_j(\alpha) \log p_j^{\star}.
\end{eqnarray}
Differentiating the numerator and denominator in Eq.~\ref{eq:gradient_alpha_fraction}, we get:
\begin{eqnarray}\label{eq:gradient_alpha_shannon_01}
    \frac{\partial p_i^{\star}}{\partial \alpha} &=&
    \lim_{\alpha \rightarrow 1^+} \frac{(1 + (\alpha-1)\HHs(\bm{p}^{\star})) \tilde{p}_i(\alpha) (\log p_i^{\star} - \sum_j \tilde{p}_j(\alpha) \log p_j^{\star}) - p_i^{\star}\log p_i^{\star} - \tilde{p}_i(\alpha) \HHs(\bm{p}^{\star})}{2(\alpha-1)} \nonumber\\
    &=& A + B,
\end{eqnarray}
with
\begin{eqnarray}\label{eq:A_shannon}
A &=& \lim_{\alpha \rightarrow 1^+} \frac{\HHs(\bm{p}^{\star}) \tilde{p}_i(\alpha) (\log p_i^{\star} - \sum_j \tilde{p}_j(\alpha) \log p_j^{\star}) \HHs(\bm{p}^{\star})}{2}\nonumber\\
&=& \frac{\HHs(\bm{p}^{\star}) p_i^{\star}\log p_i^{\star} + p_i^{\star} (\HHs(\bm{p}^{\star}))^2}{2},
\end{eqnarray}
and
\begin{equation}\label{eq:B_shannon}
B = \lim_{\alpha \rightarrow 1^+} \frac{\tilde{p}_i(\alpha) (\log p_i^{\star} - \sum_j \tilde{p}_j(\alpha) \log p_j^{\star}) - p_i^{\star}\log p_i^{\star} - \tilde{p}_i(\alpha) \HHs(\bm{p}^{\star})}{2(\alpha-1)}.
\end{equation}
When $\alpha\rightarrow 1^+$, $B$ becomes again a $\frac{0}{0}$ indetermination, which we can solve by applying again L'H\^opital's rule. Differentiating the numerator and denominator in Eq.~\ref{eq:B_shannon}:
\begin{eqnarray}\label{eq:B_shannon_02}
B &=& \frac{1}{2}\lim_{\alpha \rightarrow 1^+} \left\{ \tilde{p}_i(\alpha) \log p_i^{\star} \left(\sum_j \tilde{p}_j(\alpha) \log p_j^{\star} - \log p_i^{\star}\right) \right. \nonumber\\
&& - \tilde{p}_i(\alpha) \left(\sum_j \tilde{p}_j(\alpha) \log p_j^{\star} - \log p_i^{\star}\right) \left(\sum_j \tilde{p}_j(\alpha) \log p_j^{\star} + \HHs(\bm{p}^{\star})\right) \nonumber\\
&& \left. - \tilde{p}_i(\alpha) \sum_j \tilde{p}_j(\alpha) \log p_j^{\star} \left(\sum_k \tilde{p}_k(\alpha) \log p_k^{\star} - \log p_j^{\star}\right)\right\}\nonumber\\
&=& \frac{-p_i^{\star} \log p_i^{\star}(\HHs(\bm{p}^{\star}) + \log p_i^{\star})
+p_i^{\star} \sum_j p_j^{\star} \log p_j^{\star}(\HHs(\bm{p}^{\star}) + \log p_j^{\star})}{2}\nonumber\\
&=& \frac{-\HHs(\bm{p}^{\star}) p_i^{\star}\log p_i^{\star} - p_i^{\star} (\HHs(\bm{p}^{\star}))^2 - p_i^{\star}\log^2 p_i^{\star} + p_i^{\star} \sum_j p_j^{\star}\log^2 p_j^{\star}}{2}.
\end{eqnarray}
Finally, summing Eq.~\ref{eq:A_shannon} and Eq.~\ref{eq:B_shannon_02}, we get
\begin{eqnarray}\label{eq:gradient_alpha_shannon_02}
    \frac{\partial p_i^{\star}}{\partial \alpha}\bigg|_{\alpha=1} &=& \frac{- p_i^{\star}\log^2 p_i^{\star} + p_i^{\star} \sum_j p_j^{\star}\log^2 p_j^{\star}}{2}.
\end{eqnarray}

\subsection{Summary}

To sum up, we have the following expression for the Jacobian of $\aentmax$ with respect to $\alpha$:

\begin{equation}\label{eq:final_gradient_alpha_supp}
    \frac{\partial p_i^{\star}}{\partial \alpha} =
    \left\{
    \begin{array}{ll}
    \frac{p_i^{\star} - \tilde{p}_i(\alpha)}{(\alpha-1)^2} - \frac{p_i^{\star}\log p_i^{\star} + \tilde{p}_i(\alpha)\HHs(\bm{p}^{\star})}{\alpha-1}, & \text{for $\alpha > 1$}\\
    \frac{- p_i^{\star}\log^2 p_i^{\star} + p_i^{\star} \sum_j p_j^{\star}\log^2 p_j^{\star}}{2}, & \text{for $\alpha = 1$.}
    \end{array}
    \right.
\end{equation}

\end{document}